\documentclass{article}
\usepackage{ijcai25}
\usepackage{times}
\usepackage{soul}
\usepackage{url}
\usepackage[hidelinks]{hyperref}
\usepackage[utf8]{inputenc}
\usepackage[small]{caption}
\usepackage{graphicx}
\usepackage{amsmath}
\usepackage{amsthm}
\usepackage{booktabs}
\usepackage{cleveref}
\usepackage{enumitem}
\usepackage{mathrsfs}%
\usepackage[title]{appendix}%
\usepackage[dvipsnames,table,xcdraw]{xcolor}%
\usepackage{textcomp}%
\usepackage{manyfoot}%
\usepackage{booktabs}%
\usepackage{algpseudocode}%
\usepackage{listings}%
\usepackage{rotating}
\usepackage{comment}
\usepackage{soul}
\usepackage{array}
\usepackage{booktabs}
\usepackage{tikz}
\usetikzlibrary{shapes, arrows.meta, positioning}
\usepackage{algpseudocode}
\usepackage{bbding}
\usepackage{enumitem}
\usepackage{etoolbox}
\usepackage{graphicx}%
\usepackage{multirow}
\usepackage{multicol}
\usepackage{makecell}
\usepackage{diagbox}
\usepackage{lipsum}
\usepackage{amssymb} 
\usepackage[autostyle]{csquotes}
\usepackage{cleveref}
\usepackage{bm}
\usepackage{amsmath,amssymb,amsfonts}
\usepackage{amsthm}

\usepackage[ruled,vlined]{algorithm2e}
\usepackage{mathtools}
\newtheorem{proposition}{Proposition}[section]
\newtheorem{definition}{Definition}[section]

\newtheorem{properties}{Properties}[section]

\usetikzlibrary{shapes, arrows.meta, positioning}

\usepackage[switch]{lineno}
\usepackage{float}
\usepackage{caption}
\usepackage{etoc}
\usepackage{subcaption}

\definecolor{DeepBlue}{RGB}{0,0,139} 
\definecolor{Blue}{RGB}{46, 134, 193} 
\definecolor{LightBlue}{RGB}{133, 193, 233} 
\definecolor{PaleBlue}{RGB}{176,224,230} 

\title{A Methodological Framework for Measuring Spatial Labeling Similarity}

\author{
Yihang Du$^{1*}$\and
Jiaying Hu$^{2*}$\and
Suyang Hou$^3$\and
Yueyang Ding$^4$\and
Xiaobo Sun$^{1\dagger}$\\
\affiliations
$^1$School of Statistics and Mathematics, Zhongnan University of Economics and Law\\
$^2$Department of Biomedical Engineering, Southern University of Science and Technology\\
$^3$School of Information Engineering, Zhongnan University of Economics and Law\\
$^4$School of Life Science, Hangzhou Institute for Advanced Study, University of Chinese Academy of science\\
\emails
\{duyh,suyang\}@stu.zuel.edu.cn,
jy.hu1@siat.ac.cn,
dingyueyang24@mails.ucas.ac.cn,
xsun28@gmail.com
}

\begin{document}

\maketitle

\begin{abstract}
Spatial labeling assigns labels to specific spatial locations to characterize their spatial properties and relationships, with broad applications in scientific research and practice. Measuring the similarity between two spatial labelings is essential for understanding their differences and the contributing factors, such as changes in location properties or labeling methods. An adequate and unbiased measurement of spatial labeling similarity should consider the number of matched labels (label agreement), the topology of spatial label distribution, and the heterogeneous impacts of mismatched labels. However, existing methods often fail to account for all these aspects. To address this gap, we propose a methodological framework to guide the development of methods that meet these requirements. 
Given two spatial labelings, the framework transforms them into graphs based on location organization, labels, and attributes (e.g., location significance). The distributions of their graph attributes are then extracted, enabling an efficient computation of distributional discrepancy to reflect the dissimilarity level between the two labelings. We further provide a concrete implementation of this framework, termed Spatial Labeling Analogy Metric (SLAM), along with an analysis of its theoretical foundation, for evaluating spatial labeling results in spatial transcriptomics (ST) \textit{as per} their similarity with ground truth labeling. Through a series of carefully designed experimental cases involving both simulated and real ST data, we demonstrate that SLAM provides a comprehensive and accurate reflection of labeling quality compared to other well-established evaluation metrics. Our code is available at \href{https://github.com/YihDu/SLAM}{here}.
\end{abstract}

\renewcommand{\thefootnote}{\fnsymbol{footnote}}

\footnotetext[1]{Equal contribution.}
\footnotetext[2]{Corresponding author.}

\vspace{-1.6em}

\section{Introduction}
One of the most puzzling questions in the U.S. in 2024 is who, Trump or Harris, will step into the president office. The answer to this question may be explored through labeling each state's preference for Democrats (blue) or Republicans (red) on the election poll map before the Election Day and compare its similarity to those from 2016  and 2020  (\Cref{fig:illustration}a). If it more closely resembles the 2020 map, Harris is likely favored; otherwise, Trump may prevail. This scenario exemplifies the task of spatial labeling and its similarity measurement. 
More formally, spatial labeling involves assigning labels to specific spatial locations (a.k.a. \textit{ spatial spots}) through manual assignment, spatial classification, or clustering algorithms. For instance, in pathological analysis, spatial clustering methods assign group labels to fixed spatial spots across a tissue section\cite{ST0}, dissecting it into biologically distinct domains (\Cref{fig:illustration}b) based on spot-wise gene expression profiles detected by spatial transcriptomics (ST) technologies\cite{ST}. Similarly, in epidemiological analysis of influenza at specified geographical locations \cite{wang2019irregular}, the primary influenza type at each location serves as its spatial label.

\begin{figure*}[t]
    \centering
    \captionsetup{belowskip=-10pt}
    \includegraphics[width=0.9\textwidth]{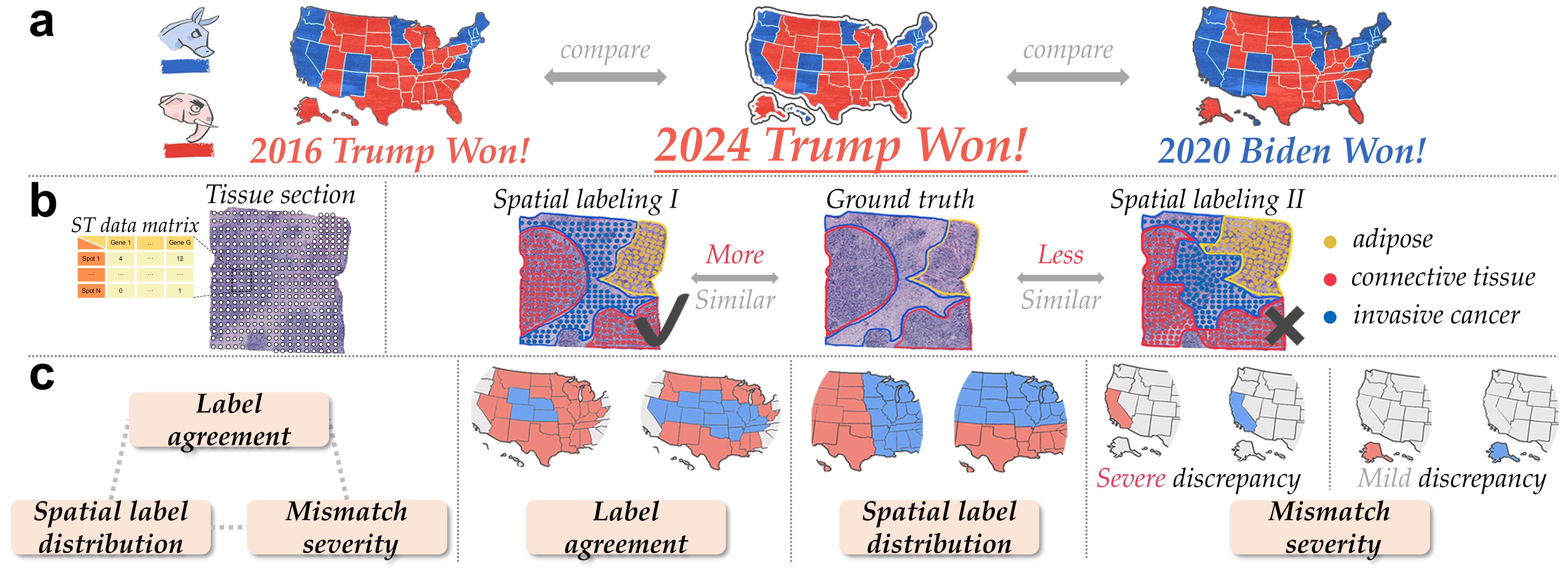}
    \caption{Spatial labeling and its similarity measurement. a, Label each state's favoring for the two parties on the poll map and compare the current map to historical ones. b, ST measures gene expression profiles at spots across a tissue section. Spatial labeling method assigns labels to spots to represent underlying biological domains. c, Three aspects in measuring spatial labeling similarity.}
    \label{fig:illustration}
\end{figure*}

In scientific research, it is common to compare two spatial labelings of identical locations by measuring their similarity. 
For example, evaluating the effectiveness of ST classification and clustering methods requires comparing spatially labeled tissue section with the ground truth labeling\cite{ST1} (\Cref{fig:illustration}b). 
Measuring similarity between spatial labels of major influenza types over time can estimate the virus' evolutionary trend and rate. Therefore, adequate measurement of spatial labeling similarity facilitates the knowledge discovery and data mining from spatial data, which should consider three aspects: label agreement, spatial label distribution, and the severity of mismatched labels. Label agreement refers to the total number of matched labels. Spatial label distribution concerns the consistency in topological structures between two spatial labelings, implying that similarity can vary with spatial organization of matches and mismatches. Mismatch severity highlights the heterogeneity in mismatches' impact on overall similarity. 
Using faked election poll maps as an illustration (\Cref{fig:illustration}c), label agreement indicates the number of states favoring the same party in both maps. The impact of spatial label distribution can be seen in the dissimilarity between a horizontally split map (northern states for Democrats versus southern states for Republicans) and a vertically split map (eastern states for Democrats versus western states for Republicans). Mismatch severity means that differing in Alaska's color (3 electoral votes) is less significant than differing in California's color (55 electoral votes).

Existing methods for evaluating spatial labeling similarity can be broadly divided into two categories. 
The first focuses on label agreement and includes two types of methods designed for labels generated in classification and clustering tasks, referred to as \textit{supervised} and \textit{unsupervised}, respectively. Mathematically, assume that we have two spatial labeling results over $n$ spots: $R_1=\{r:r\in L_1=\{a_1,a_2,..a_k\}\}^n$ and $R_2=\{r:r\in L_2=\{b_1,b_2,..b_{k'}\}\}^n$, where $L_1$ and $L_2$ represent their respective label space. Supervised methods, applicable when the two label spaces are identical (i.e., $L_1\equiv L_2$), include accuracy, precision, recall, and F1 score. 
Unsupervised methods, used when $L_1\not\equiv L_2$ (e.g., cluster labels assigned by clustering method versus ground truth labels), include the Adjusted Rand Index (ARI) \cite{ARI}, Normalized Mutual Information (NMI) \cite{NMI}, Jaccard Score \cite{JaccardScore}, V-measure \cite{v-measure}, and Fowlkes-Mallows Index (FMI) \cite{FMI}. They typically work by counting agreements and disagreements between pairwise spot labels. However, a significant limitation of both supervised and unsupervised methods is that they treat spatial spots independently, overlooking their spatial relationships. Consequently, they can incorrectly regard two spatial labeling results with distinct topological organizations as having high similarity. The second category comprises statistical tests like the Weisfeiler-Lehman (WL) test \cite{weisfeilerlehmantest} and the permutation test \cite{thepermutationtest}, which are purposed for determining statistical significance of similarity without quantifying similarity level, rendering them unsuitable for this study. 

Additionally, both categories treat mismatched labels equally, which, however, can differently impact the labeling similarity in practice. However, this limitation cannot be trivially addressed by assigning weights to spots because the spatial position of mismatched labels also matters. For example, in the electoral map example, Georgia and Michigan have identical weights (16 votes). However, Georgia’s shift from red to blue holds more significance due to its geographical connection to other red states, signaling potential political disunity.  These limitations collectively restrict the use of these methods to obtain a comprehensive, sensible, and unbiased result.

To address these limitations, we propose a methodological framework for measuring spatial labeling similarity, accounting for label agreement, spatial label distribution, and mismatching severity. The framework's workflow comprises four steps (\Cref{fig:fig2}): Initially, if the two spatial labelings have different label spaces, they are matched to a common one using a matching function (\textbf{Step I}). Next, we construct a basic graph, where each spot is represented by a node positioned according to its location, with spot attributes (e.g., gene expression level) serving as node attribute. Edge types and weights are set based on the labels and attributes of the connecting nodes using a graph edit function, thereby incorporating the label matching degree and mismatching severity (\textbf{Step II}). In\textbf{ Step III}, the distribution of graph attributes is estimated using an attribute extracting function. The dissimilarity of two spatial labelings is then represented by the expected discrepancy between their graph attribute distributions using a discrepancy function(\textbf{Step IV}). To demonstrate the efficacy of this framework, we implement it as Spatial Labeling Analogy Metric (\textbf{SLAM}), a novel metric for evaluating the quality of spatial labeling results and test it using both simulated and real ST data. Our main contributions include:

\begin{itemize}[topsep = 2pt, parsep = 0pt , left=0pt]
    \item \textbf{Methodological Framework and Implementation}. We propose a methodological framework for developing comprehensive and unbiased methods for measuring spatial labeling similarity, addressing label agreement, spatial label distribution, and mismatching severity which existing methods fail to achieve simultaneously. We further implement the framework as SLAM, an innovative spatial labeling evaluation metric, along with an analysis for its theoretical foundation (Appendix A.2 and A.3).
    \item \textbf{Experimental Validation}. Through deliberately designed experimental scenarios, we demonstrate that the evaluation outcomes using SLAM provide the most accurate reflection of the spatial labeling quality in ST, compared to existing metrics.
\end{itemize}

\section{Related Work}
\subsection{Label Matching-Based Methods}\label{related work label matching-based}
Labeling matching-based methods focus on the number of matched and mismatched labels between two spatial labelings, and can be divided into two categories which we refer to as \textit{supervised} or \textit{unsupervised}. Supervised methods require that two spatial labelings share the same label space, as explained in the Introduction section. Prevalent supervised methods include accuracy, precision, recall, and F1 score. In these methods, one spatial labeling is treated as the reference, and mismatched labels in the other labeling are considered errors to calculate the corresponding metrics. Unsupervised methods, on the other hand, do not require matched label space, instead emphasizing pairwise label agreement. That is, two locations agree if they either share the same label or not in both spatial labeling results. 
Common unsupervised methods include ARI, NMI, Jaccard Score, V-measure, and FMI. The ARI is an improved version of Rand index (RI), which calculates the ratio of location pairs that consistently share label or not to all possible location pairs in two spatial labelings. NMI normalizes the mutual information between two spatial labelings using their average entropies, accounting for different labeling dispersion degree and numbers, with a larger value indicating greater similarity. V-measure is equivalent to NMI, using arithmetic mean as the aggregation function. The Jaccard Score is computed as the ratio of the intersection to the union of the two sets of label pairs created from two spatial labelings. FMI takes one spatial labeling as reference and counts true positives (TPs), false positives (FPs), and false negatives (FNs) based on label agreement between the two spatial labelings. It is calculated as the geometric mean of precision and recall from TPs, FPs, and FNs. While these metrics highlight different aspects of the agreement between the two spatial labelings, they all assume that locations are equal and independent, overlooking spatial relationships and varying significance among locations, thus leading to biased evaluation results. The detailed calculation of these methods can be found in Appendix F.

\subsection{Statistical Test-Based Methods}
Statistical test-based methods (e.g., WL test) typically involve computing a statistic under the null hypothesis that the two spatial labelings are independent and dissimilar. If the observed statistic value significantly deviates from the expected value, we reject the null hypothesis and believe the two spatial labeling are significantly more similar than expected by chance. 
The major drawback of these methods lies in their inability to quantify and compare similarity levels across labeling pairs, making them unsuitable for the objective of this study.

\section{Method}

\subsection{Methodological Framework}

Let $Y^{(1)} \in \{a_1 \cdots, a_{K_1}\}^n$ and $\widehat{Y}^{(2)} \in \{b_2, \cdots, b_{K_2}\}^n$ denote the label vectors of the two spatial labelings to be compared, where $K_1$ and $K_2$ are the total numbers of distinct labels in $Y^{(1)}$ and $\widehat{Y}^{(2)}$, respectively. If $Y^{(1)}$ and $\widehat{Y}^{(2)}$ are in distinct label space (e.g., $K_1\neq K_2$), $\widehat{Y}^{(2)}$ is mapped to the space of $Y^{(1)}$ using a matching function $\mathcal{M}:b_v\rightarrow a_u$: 
\begin{equation}\label{matching function}
    Y^{(2)} = \mathcal{M}(\widehat{Y}^{(2)},Y^{(1)}).
\end{equation}
Let $G^s(V, E^s)$ represent the basic graph in Step II, where $V$ represents the set of nodes and $E^s\in \{1,0\}^{|E^S|}$ the set of edges. Given a spatial labeling result, its specific graph is constructed using a graph attribute editing function $\mathcal{G}$, which incorporates the label distribution and, if available, spot attributes $X \in \mathbb{R}^{|V| \times d}$, where $d$ denotes the attribute dimension:
\begin{equation}
    G^{(*)}(V, E^{(*)},W^{(*)}) = \mathcal{G}(X,Y^{(*)},G^s),\ \text{where}\ * \in \{1,2\}.
\end{equation}

Here, $W^{(*)}\in \mathbb{R}^{|E^{(*)}|}$ denotes the set of edge weights computed \textit{as per} node labels (and attributes). The similarity between the two spatial labelings is estimated by comparing their graph attribute distributions, $ Z^{(1)}$ and $ Z^{(2)}$, obtained using an attribute extraction function $\mathcal{T}$:
\begin{equation}\label{attribute function}
    Z^{(*)} \in \mathbb{R}^P = \mathcal{T}(G^{(*)}),\ \text{where}\ * \in \{1,2\}.
\end{equation}

The distributional discrepancy between $ Z^{(1)}$ and $ Z^{(2)}$ is computed using a discrepancy function $\mathcal{D}:\mathbb{R}^P\rightarrow\mathbb{R}$:
\begin{equation}\label{discrepancy function}
    d = \mathcal{D}(Z^{(1)},Z^{(2)}).
\end{equation}
Here, $d$ is a discrepancy score, where a lower value indicates higher similarity between the two spatial labelings. Note that there are multiple choices for $\mathcal{M}$, $\mathcal{G}$, $\mathcal{T}$, and $\mathcal{D}$ \textit{as per} the specific context. For example, $\mathcal{M}$ can be the Hungarian algorithm \cite{Hungarian} or kernel-based probability density methods \cite{mars};  ;$\mathcal{T}$ can be functions for calculating graph coefficients \cite{cc} or Laplacian spectrum \cite{spectral}; while $\mathcal{D}$ can be functions for calculating the maximum-mean discrepancy \cite{MMD} or Wasserstein distance \cite{rabin2012wasserstein}. In the following sections, we will give an implementation of our framework for evaluating the spatial labeling results using ST data as an example.



\begin{figure*}[t]
    \centering
    \captionsetup{belowskip=-10pt}
    \includegraphics[width=0.85\linewidth]{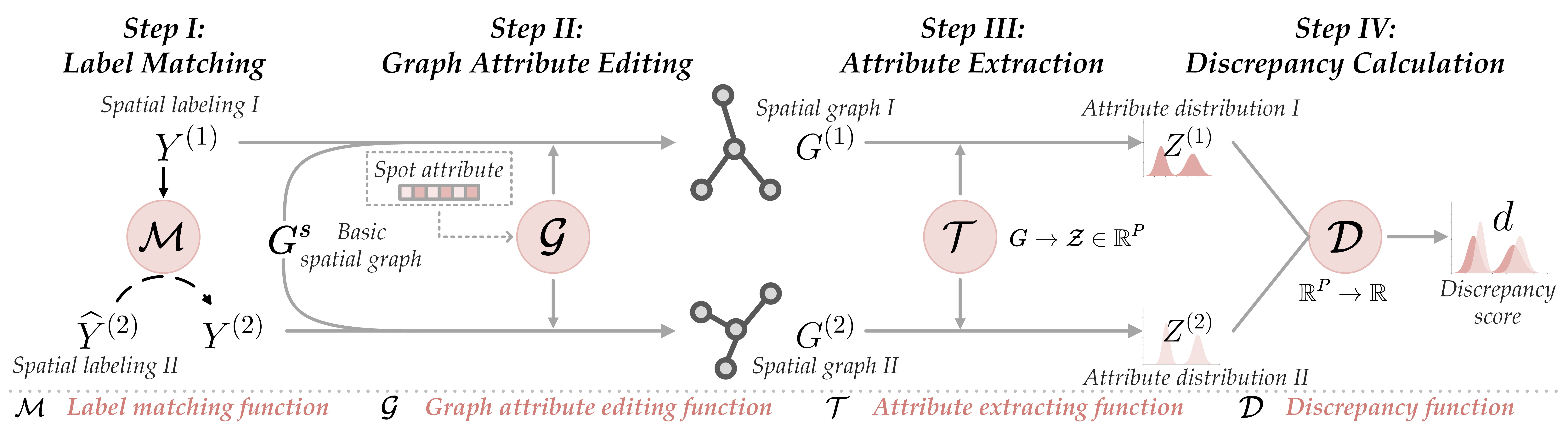}
    \caption{Overview of the methodological framework.}
    \label{fig:fig2}
\end{figure*}

\subsection{Application: Evaluating Spatial Labeling Result in Spatial Transcriptomics with SLAM}
A common application that involves measuring spatial labeling similarity is the evaluation of spatial labeling result against ground truth labeling. To this end, we implement our framework as SLAM and test it using ST data. Spatial labeling in ST involves assigning spot labels, representing biologically distinct tissue domains, based on their locations and gene expression profiles. Specifically, an ST dataset is represented as $X \in \mathbb{R}^{n \times g}$, where $n$ is the number of spots across the tissue section, $g$ is the number of genes, and $x_{i,j}$ denotes gene $j$'s expression at location $i$. We define the ground truth labels as $Y^{(0)} \in \{1,2, \cdots, K\}^n$ and a spatial labeling result as $\widehat{Y}^{(1)} \in \{a_1,a_2..a_{K_1}\}^{n}$, contains $K_1$ groups. 

\subsubsection{Label matching} 
We implement the $\mathcal{M}$ function in \Cref{matching function} to match $\widehat{Y}^{(1)}$ with $Y^{(0)}$: 

\begin{equation}
    Y^{(1)} = \mathcal{M}(\widehat{Y}^{(1)},Y^{(0)}),
\end{equation}

where $Y^{(1)}\in \{1,2, \cdots, K\}^n$ is the post-matching labels. $\mathcal{M}$ denotes our label matching function. Here, we propose a Jaccard coefficient-based matching algorithm as $\mathcal{M}$, which can work even if $K_1 \neq K$ (see Appendix A.1 for details). 

\subsubsection{Constructing basic spatial graph} \label{basic spatial graph construction}
Given a ST dataset, the basic spatial graph $G^s(V, E^s)$ is constructed with the node set $V$ comprising its spatial spots. 
In ST data such as 10x Visium, spots are typically arranged on a hexagonal grid covering the tissue slice. Thus, their spatial relationship can be naturally modeled using a mutual k-nearest neighbor graph(\cite{STAGATE,zong2022const}), where the edge set $E^s$ is defined as:
\begin{equation}
    E^s:=\{(u,v)|u \in N_k(v) \text{ and } v \in N_k(u), \forall u,v \in V\},
\end{equation}
where $N_k(v)$ denotes the $k$-nearest neighbor set of node $v$. We also define a function $I(u,v):(u,v)\mapsto \{1, 2, \cdots, |E^s|\}$ that maps connected node pairs to edge indices.




\subsubsection{Constructing label-conditional attributed graph}
We generate labeling-specific graph $G^{(i)}, i \in \{0, 1\}$, conditional on $G^s$, node labels $Y^{(i)}$, and gene expression profiles $X$ using a graph attribute edit function $\mathcal{G}$:
\begin{equation}
    G^{(i)}(V, E^{(i)}, W^{(i)}) = \mathcal{G}(Y^{(i)},G^s, X).
\end{equation}
Here, the types of edge $E^{(i)}$ are defined as:
\begin{equation}
type(E^{(i)}_{I(u,v)}) =
\begin{cases}
t, & \text{if }  y_{u}^{(i)} = y_{v}^{(i)} = t, \\
0, & \text{otherwise.}
\end{cases}, \forall (u,v)\in E^s,
\end{equation}
where $t \in \{1, \cdots, K\}$. 
Additionally, $W^{(i)}\in \mathbb{R}^{|E^{(i)}|}$ is introduced to account for severity variations across mislabel types. For example, a pair of similar spots of the same type should be penalized more if they are assigned distinct labels, while a pair of dissimilar spots of different types should be penalized more if they are assigned the same label. Then, we have:  
\begin{equation} \label{gene similarity}
\resizebox{\linewidth}{!}{$
W^{(i)}_{I(u,v)} = 
\begin{cases}
1-Sim(x_u, x_v), \text{if}\ type(E^{(0)}_{I(u,v)})=0,\\
Sim(x_u, x_v), \text{if } type(E^{(0)}_{I(u,v)})\neq 0 \ \\
\end{cases}, \forall (u,v)\in E^s,
$}
\end{equation}
where $x_u,x_v\sim \mathbb{R}^g\in X$. $Sim(x_u, x_v)$ is a function measuring the similarity between attributes of nodes $u$ and $v$.

\subsubsection{Extracting graph attribute distribution} 
Due to the randomness inherent to spatial labeling method, the observed annotation result $Y^{(i)} $ merely represents a sample from an unknown underlying distribution $f^{(i)}$. Therefore, we aim to estimate $f^{(i)}$ given the label-conditional attributed graph $ G^{(i)}(V, E^{(i)}, W^{(i)})$. 
Formally, we extract edge attributes using the function $\mathcal{T}$ in \Cref{attribute function}, which is a one-hot encoding function here:

\begin{equation}\label{one hot encoding equation}
    \begin{aligned}
        z_{r,k}^{(i)} &= \mathcal{T}(E_r^{(i)}) = 
        \begin{cases}
            1, & \text{if } type(E_{r}^{(i)}) = k,  \\
            0, & \text{otherwise.}
        \end{cases}, \\ & \forall k \in \{1, \cdots, K\}, \forall r \in \{1, 2, \dots, |E^{(i)}|\}
    \end{aligned}
\end{equation}
 
The edge attribute matrix of $G^{(i)}$ can be represented as $ Z^{(i)} :=[z_1^{(i)},...,z_{|E^{(i)}|}^{(i)}] ^T\in \{0,1\}^{|E^{(i)}|\times K}$, where $ z_{r}^{(i)} :=[z_{r,1}^{(i)},...,z_{r,K}^{(i)}]^T$. $Z^{(i)}$ is then adjusted with the edge weight matrix $\widehat{W}^{(i)}:=[vec(W^{(i)})]^K\in\mathbb{R}^{|E^{(i)}|\times K}$ to account for mismatch severity:

\begin{equation} \label{severity weight}
    Z^{(i)}= Z^{(i)}\odot\widehat{W}^{(i)}.
\end{equation}
The density $f^{(i)}$ of edge attribute distribution $ Z^{(i)}$ is then estimated using a Gaussian kernel density estimator $\mathcal{K}$: 
\begin{equation}
    f^{(i)}(x)=\frac{1}{|E^{(i)}|\times h^K}\sum_{r=1}^{|E^{(i)}|}\mathcal{K}( \frac{x - z_r^{(i)}}{h}), 
    \label{KDE}
\end{equation}
where $h$ is the bandwidth for which a sensitivity analysis is also conducted in Appendix E.

\subsubsection{Computing the discrepancy between graph attribute distributions}
We implement the discrepancy function $\mathcal{D}$ in \Cref{discrepancy function} as a composite function of a sliced Wasserstein distance function $\mathcal{W}$, a symmetric positive definite exponential kernel function $\Xi$, and a maximum-mean discrepancy (MMD) function $\Delta$. The sliced Wasserstein distance function measures the discrepancy between two sample distributions of graph attributes. Since the calculation of one-dimensional Wasserstein distance has closed-form solution, sliced Wasserstein distance greatly improves the computational efficiency for multivariates \cite{bonneel2015sliced}, which is also empirically validated in our complexity analysis in Appendix D. 
The kernel function $\Xi \circ \mathcal{W}^2$ is a symmetric positive definite kernel and captures higher-order moments of distributional discrepancy in a uniquely-induced reproducing kernel Hilbert space (RKHS), for which we provide a detailed mathematical proof in Appendix A.2 and A.3.
In this uniquely-induced RKHS, the MMD function $\Delta$ computes the expected discrepancy between the two underlying graph attribute distributions. Put together, we have $\mathcal{D}:=(\Delta^2\circ \Xi\circ \mathcal{W}^2)$ and compute the discrepancy between the underlying graph attribute distributions of the clustering result, $ f^{(1)}$, and the ground truth, $f^{(0)}$, as $d$ \cite{MMD}:
\begin{equation}
\resizebox{\linewidth}{!}{$
\begin{aligned}
     & d\in [0,2]=(\Delta^2\circ \Xi\circ \mathcal{W}^2)[ f^{(0)}|| f^{(1)}] =  E_{x,x'\sim f^{(0)}}[\Xi(\mathcal{W}^2(x, x'))] \\ 
     &+ E_{y, y'\sim f^{(1)}}[\Xi(\mathcal{W}^2(y, y'))] - 2E_{x\sim f^{(0)},y\sim f^{(1)}}[\Xi(\mathcal{W}^2(x, y))],
\end{aligned}
$}
\end{equation}
where $x,y\in \mathbb{R}^K$. $d$  is then reduced to:
\begin{equation}
    \resizebox{\linewidth}{!}{$
    \begin{aligned}
        & d= \\
        & \frac{1}{n_{0}^2} \sum_{i=1}^{n_{0}} \sum_{i'=i+1}^{n_{0}} \Xi(\mathcal{W}^2(x_i, x_{i'})) + \frac{1}{n_{1}^2} \sum_{j=1}^{n_{1}} \sum_{j'=j+1}^{n_{1}} \Xi(\mathcal{W}^2(y_j, y_{j'})) \\
        & - \frac{2}{n_{0}n_{1}} \sum_{i=1}^{n_{0}} \sum_{j=1}^{n_{1}} \Xi(\mathcal{W}^2(x_i, y_j)),
    \end{aligned}
    $}
\end{equation}
\begin{equation}
    \Xi(\mathcal{W}^2((x_i, x_{i'})) = \exp\left(-\gamma\mathcal{W}^2(x_i, x_{i'})\right), \gamma >0
\end{equation}
where $n_0$ and $n_1$ represent the number of sampled distributions from $f^{(0)}$ and $f^{(1)}$, respectively.

\section{Experimental Design} 
We design seven experimental cases in which a variety of spatial labeling results of simulated and real ST datasets are evaluated using SLAM and fourteen benchmark metrics. 




\subsection{Experimental Cases} 
\paragraph{Label Agreement Degree.}
\textit{\textbf{Case I}: Topologically identical spatial labelings with different numbers of mislabels.} An effective evaluation metric should reflect the quality change in spatial labeling due to an altered number of mislabels, even when the topological structure of the labels remains unchanged.
\textit{\textbf{Case II}: Increased number of mislabels} (Appendix C).

\paragraph{Consistency in Spatial Label Distribution.} \textit{\textbf{Case III}: Mislabels at the center versus periphery.} In the ground truth labeling of a tissue section in ST, spots at domain center are typically more definitive in their types than those at the domain periphery. Therefore, mislabeling domain center spots represents a more severe error with greater dissimilarity with the ground truth, which should be reflected by an effective evaluation metric. \textit{\textbf{Case IV}: Aggregated versus dispersed mislabels} (Appendix C).


\paragraph{Mislabeling Severity.}\label{mislabeling Severity} \textit{\textbf{Case V}: False positive mislabels versus false negative mislabels.} In clinical diagnosis, false negatives (mislabeling a diseased spot as normal) are more serious than false positive (mislabeling a normal spot as diseased), leading to lower labeling quality. Thus, the evaluation metric should account for this asymmetry. \textit{\textbf{Case VI}: Mislabels with varying similarity to true labels}.Mislabeling a spot to a biologically similar type (e.g., similar in gene expression profiles) is more acceptable than mislabeling it to a biologically distinct type. The evaluation metric should give a quality score consistent with this error severity gap.

\subsubsection{Evaluating Spatial Labeling Using Real Spatial Transcriptomics Data} \textit{\textbf{Case VII}: Evaluating spatial labelings in a human breast cancer tissue section.} To evaluate SLAM's effectiveness in practice, three unsupervised methods, including GraphST \cite{Graphst}, SpaGCN \cite{spagcn}, and STAGATE \cite{STAGATE}, are employed to label a complex breast cancer tissue section(10x-hBC-A1). SLAM evaluates these labeling results by measuring their similarity to the expert-curated ground truth labels. 


\subsection{Experimental Settings}
\paragraph{Data preparation.}
To simplify experimental evaluation and result interpretation, we use NetworkX \cite{NetworkX} to simulate the graph structures of spatial labeling results and the ground truth for all simulated cases, mimicking potential results encountered in the analysis of ST data from disease tissues. 

To increase the reality of our simulated data, we use a human breast cancer 10x Visium dataset (10x-hBC-H) to match real spots to graph nodes and use real spatial gene expression to compute node similarity. For \textbf{\textit{Case VI}}, we select 10 spots from each of the breast glands, adipose, and cancer regions in the human breast cancer 10x Visium dataset (10x-hBC-H1). In \textbf{\textit{Case VII}}, where real spatial clustering results are obtained by applying real spatial clustering methods to the 10x-hBC-A1 dataset, we follow the workflow guides provided in the original studies. 

\paragraph{Benchmark methods.}\label{benchmark section}
In each experimental case, we use the evaluation metrics suggested by \cite{benchmarks} as benchmarks. These include the supervised metrics (accuracy, precision, recall, and F1 score) and the unsupervised metrics (ARI, NMI, Jaccard Score, V-measure, and FMI) described in \Cref{related work label matching-based}. Additionally, we include five internal metrics, which differ from SLAM and other benchmark metrics in that they do not compare the spatial labeling results with the ground truth. Instead, they evaluate the results internally based on same-label cohesion and distinct-label separation. Their inclusion ensures the comprehensiveness of our benchmarks. The calculations, ranges, and directions of all these metrics are detailed in Appendix F and Appendix G.

\paragraph{Evaluating SLAM and benchmark methods.}

We evaluate SLAM and benchmark methods from two perspectives: consistency and sensitivity. Consistency means that changes in output values should align with changes in spatial labeling quality. Sensitivity refers to the method’s ability to reflect changes in spatial labeling quality. To achieve this, we design a \textbf{\textit{Q} coefficient}, where a large positive value indicates that the corresponding metric is both consistent (positiveness) and sensitive (large value change) to changes in spatial labeling quality.  Refer to Appendix B for details.


\section{Results}
Results of case II and IV are put in Appendix C. 
\subsection{Label Agreement Degree}


\begin{table*}[t]
\centering
\caption{\textit{Q} coefficient of SLAM and benchmark metrics in Cases \textit{I}, \textit{III}, \textit{IV\textsuperscript{*}},\textit{V} and \textit{VI}(\textsuperscript{*} See Appendix C). A positive \textit{Q} value (in \textcolor{OliveGreen}{green}) indicates the evaluation metric faithfully reflects the labeling quality change. Otherwise, the value is indicated in \textcolor{Red}{red}. N/A indicates that the metric is not applicable. Only SLAM appropriately works for all cases.}
\vspace{-8pt}
\label{tab:Table 1}
\resizebox{\textwidth}{!}{
\begin{tabular}{c|ccccccccccccccc}
\toprule
\multirow{2}{*}{\textit{Case-ID}} & \multicolumn{1}{c}{\multirow{2}{*}{\textbf{SLAM}}} & \multicolumn{5}{|c|}{\textit{Unsupervised External}} & \multicolumn{5}{c|}{\textit{Unsupervised Internal}} & \multicolumn{4}{c}{\textit{Supervised}} \\
& & \multicolumn{1}{|c}{ARI} & NMI & Jaccard Score & FMI & \multicolumn{1}{c|}{V-measure} & ASW & CHAOS & PAS & CH Index & \multicolumn{1}{c|}{DB Index} & Accuracy & Precision & Recall & F1 score\\
\midrule
\rowcolor{gray!10} \multicolumn{16}{c}{\textit{Label agreement degree}} \\

\textit{I}  & \textcolor{OliveGreen}{0.257} & \textcolor{Red}{0} & \textcolor{Red}{N/A} & \textcolor{OliveGreen}{0.166} & \textcolor{Red}{0} & \textcolor{Red}{N/A} & \textcolor{Red}{0} & \textcolor{Red}{0} & \textcolor{Red}{0} & \textcolor{Red}{0} & \textcolor{Red}{0}
& \textcolor{OliveGreen}{0.333} & \textcolor{Red}{N/A} & \textcolor{Red}{N/A} & \textcolor{Red}{N/A}  \\
\rowcolor{gray!10} \multicolumn{16}{c}{\textit{Consistency in spatial label distribution}} \\
\textit{III} & \textcolor{OliveGreen}{0.103} & \textcolor{Red}{0} & \textcolor{Red}{0} & \textcolor{Red}{0} & \textcolor{Red}{0} & \textcolor{Red}{0}  & \textcolor{OliveGreen}{0.056} & \textcolor{OliveGreen}{0.059} & \textcolor{OliveGreen}{0.133} & \textcolor{OliveGreen}{0.382} & \textcolor{OliveGreen}{0.224}  & \textcolor{Red}{0}  & \textcolor{Red}{0} & \textcolor{Red}{0} & \textcolor{Red}{0}\\
\textit{IV\textsuperscript{*}} & \textcolor{OliveGreen}{0.078} & \textcolor{Red}{0} & \textcolor{Red}{N/A} & \textcolor{Red}{0} & \textcolor{Red}{0} & \textcolor{Red}{N/A} & \textcolor{OliveGreen}{0.125} & \textcolor{OliveGreen}{0.040} & \textcolor{OliveGreen}{0.340} 
& \textcolor{OliveGreen}{0.988} & \textcolor{OliveGreen}{0.900} & \textcolor{Red}{0} & \textcolor{Red}{N/A} & \textcolor{Red}{N/A} & \textcolor{Red}{N/A}\\
\rowcolor{gray!10} \multicolumn{16}{c}{\textit{Mislabeling severity}} \\
\textit{V} & \textcolor{OliveGreen}{0.110} & \textcolor{Red}{0} & \textcolor{Red}{0} & \textcolor{Red}{0} & \textcolor{Red}{0} & \textcolor{Red}{0} & \textcolor{Red}{0} & \textcolor{Red}{0} & \textcolor{Red}{0} 
& \textcolor{Red}{0} & \textcolor{Red}{0} & \textcolor{Red}{0} & \textcolor{Red}{0} & \textcolor{Red}{0} & \textcolor{Red}{0}\\

\textit{VI}& \textcolor{OliveGreen}{0.073} & \textcolor{Red}{0} & \textcolor{Red}{0} & \textcolor{Red}{0} & \textcolor{Red}{0} & \textcolor{Red}{0} & \textcolor{Red}{0} & \textcolor{Red}{0} & \textcolor{Red}{0}  & \textcolor{Red}{0} & \textcolor{Red}{0} & \textcolor{Red}{0} & \textcolor{Red}{0} & \textcolor{Red}{0} & \textcolor{Red}{0}\\
\bottomrule
\end{tabular}
}
\end{table*}

\paragraph{\textbf{SLAM captures mislabel-induced quality change between topologically identical spatial labelings (\textit{Case I}).}} We simulate a dataset of 36 type A spots. Spatial labeling I mislabels 24 type A spots as type B spots on the right side, while spatial labeling II mislabels 12 type A spots as type B spots on the left side (\Cref{fig:case1}). Despite sharing the same topological structure without considering label types, spatial labeling I and II exhibit different mislabel quantities. As a result, all internal metrics remain unchanged between the two labelings, failing to reflect the labeling quality difference due to the number of mislabels (\Cref{tab:Table 1}). Conversely, SLAM, along with the supervised and unsupervised metrics, which are sensitive to the number of mislabels, correctly demonstrate the superiority of spatial labeling II over spatial labeling I. These results indicate SLAM’s effectiveness in detecting mislabel-induced changes in result’s similarity with the ground truth.

\begin{figure}[h]
    \centering
    \captionsetup{belowskip=-10pt}
    \includegraphics[width=\linewidth]{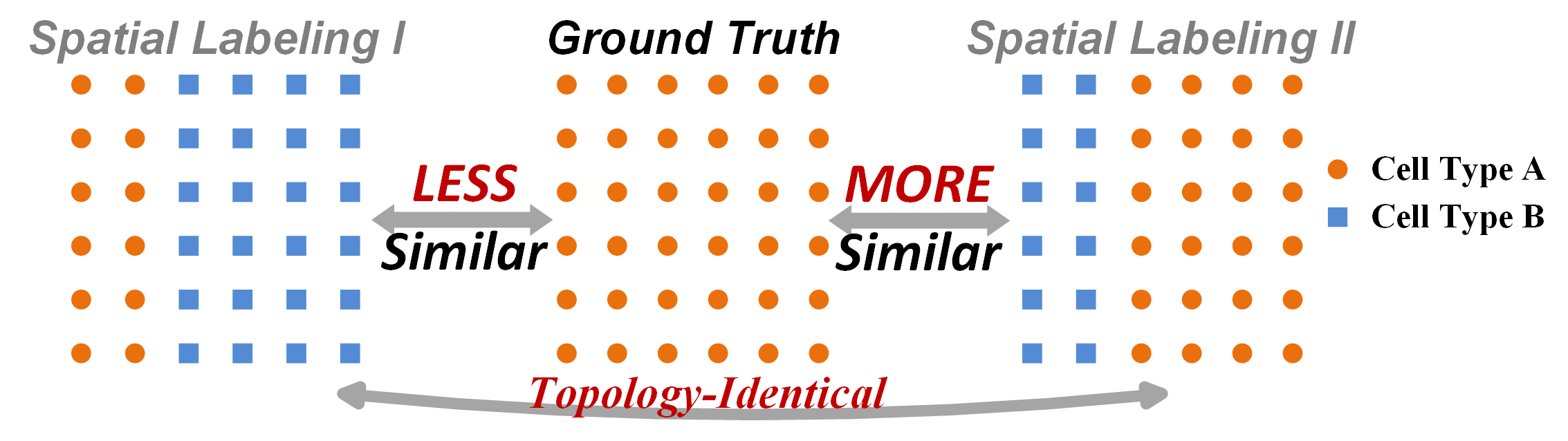}
    \caption{\textit{Case I}. SLAM reflects the difference in similarity to the ground truth labeling between two spatial labelings that have the same topological structure but differ in the number of mislabels.}
    \label{fig:case1}
\end{figure}

\subsection{Consistency in Spatial Label distribution}
\paragraph{\textbf{SLAM differentiates mislabels at core versus periphery regions (\textit{Case III)}.}}
In pathological diagnosis, tumor core region encompasses more definitive tumor spots, whereas the tumor edge region includes plausible tumor spots that resemble adjacent normal tissues. Consequently, mislabeling tumor spots as normal in the core region is more severe and dissimilar to the ground truth compared to mislabeling those at the tumor edge. In this case, we simulate a dataset of 30 spots, with 15 circles and 15 squares representing tumor and normal spots, respectively. As shown in \Cref{fig:case3}, three tumor core spots are mislabeled in spatial labeling I, while three tumor edge spots are mislabeled in spatial labeling II. All supervised and unsupervised benchmark metrics demonstrate unchanged scores for the two labelings due to their insensitivity to the topological change of labels (\Cref{tab:Table 1}). In contrast, SLAM and the five internal metrics exhibit a positive \textit{Q} coefficient, indicating their ability to correctly capture the quality gap between the two spatial labelings due to changes in mislabeling locations.

\begin{figure}[h]
    \centering
    \captionsetup{belowskip=-10pt}
    \includegraphics[width=\linewidth]{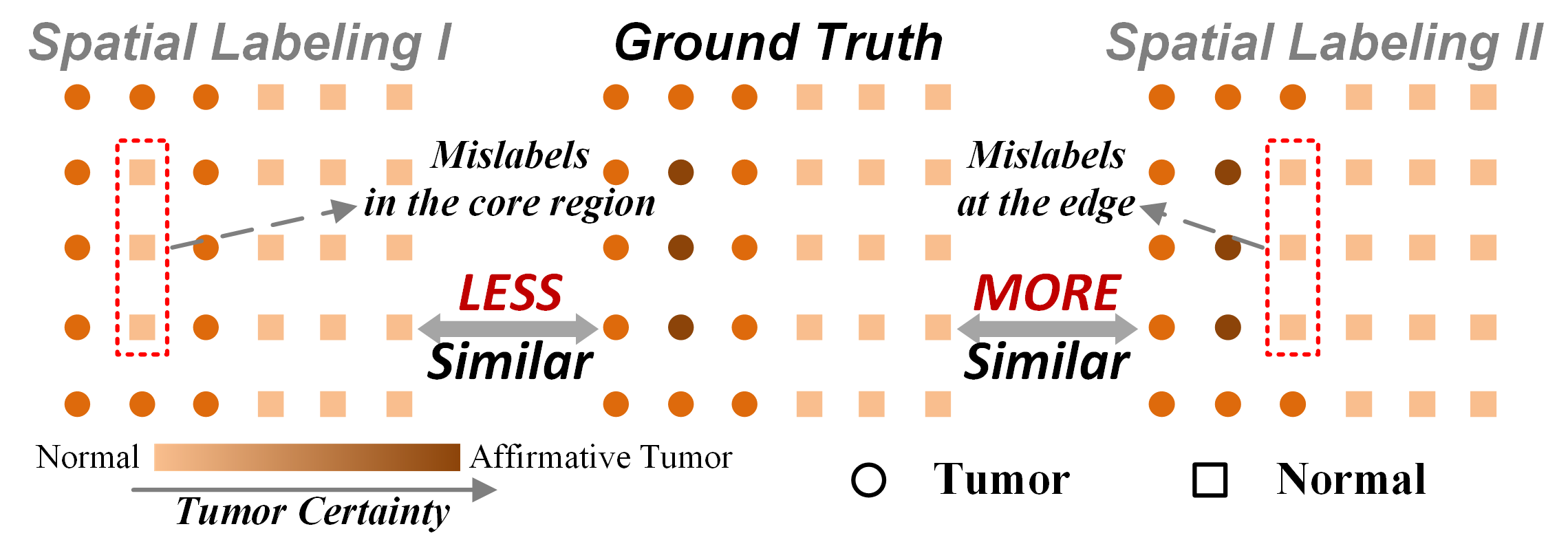}
    \caption{\textit{Case III}. In the ground truth, circles represent tumor spots and squares represent normal spots. The color bar indicates the certainty of being a tumor, with the leftmost color corresponding to normal spots and the rightmost color to affirmative tumor spots. 
    }
    \label{fig:case3}
\end{figure}

\begin{table*}[h]
\centering
\caption{Evaluation of three unsupervised spatial labeling methods using Human Breast Cancer Dataset. The ranking of these methods are indicated after their names, with \textcolor{DeepBlue}{deepblue}, \textcolor{Blue}{blue}, \textcolor{LightBlue}{lightblue} colors corresponding to \textcolor{DeepBlue}{1st}, \textcolor{Blue}{2nd}, and \textcolor{LightBlue}{3rd}, respectively. ($\uparrow/\downarrow$: Higher/lower values indicate better performance.) 
}
\vspace{-8pt}
\label{tab:Table 2}
\resizebox{\textwidth}{!}{
\begin{tabular}{l|c|ccccc|ccccc}
\toprule
Methods & \textbf{SLAM} $\downarrow$ & ARI$\uparrow$ & NMI $\uparrow$ & Jaccard Score $\uparrow$ & FMI $\uparrow$ & V-measure $\uparrow$ & ASW $\uparrow$ & CHAOS $\downarrow$  & PAS $\downarrow$ & CH index $\uparrow$ & DB index $\downarrow$  \\
\midrule
\textcolor{DeepBlue}{GraphST} & \textcolor{DeepBlue}{1.002} &  \textcolor{Blue}{0.121} &  \textcolor{Blue}{0.219} & \textcolor{Blue}{0.195} & \textcolor{Blue}{0.531} & \textcolor{Blue}{0.219} & \textcolor{LightBlue}{-0.111} & \textcolor{Blue}{0.202} & \textcolor{Blue}{0.488} & \textcolor{Blue}{25.290} & \textcolor{Blue}{6.885} \\
\textcolor{Blue}{STAGATE} & \textcolor{Blue}{1.134} & \textcolor{DeepBlue}{0.164} & \textcolor{DeepBlue}{0.221} & \textcolor{DeepBlue}{0.207} & \textcolor{DeepBlue}{0.544} & \textcolor{DeepBlue}{0.221} & \textcolor{Blue}{-0.104} & \textcolor{DeepBlue}{0.191} & \textcolor{DeepBlue}{0.390} & \textcolor{DeepBlue}{28.506} & \textcolor{DeepBlue}{3.226} \\
\textcolor{LightBlue}{SpaGCN} & \textcolor{LightBlue}{1.680} &  \textcolor{LightBlue}{0.032} &  \textcolor{LightBlue}{0.180} & \textcolor{LightBlue}{0.143} & \textcolor{LightBlue}{0.372} & \textcolor{LightBlue}{0.180} & \textcolor{DeepBlue}{-0.085} & \textcolor{LightBlue}{0.240} & \textcolor{LightBlue}{0.775} & \textcolor{LightBlue}{7.224} & \textcolor{LightBlue}{12.880} \\
\bottomrule
\end{tabular}}
\end{table*}

\subsection{Mislabeling Severity}
\paragraph{\textbf{SLAM differentiates false positive and false negative errors (\textit{Case V}).}} 
As explained in \Cref{mislabeling Severity}, false negatives (FNs) represent a more severe error than false positives (FPs) in clinical setting. This severity gap results in a difference in similarity with the ground truth labeling. Here, we assess whether SLAM and the benchmark metrics can capture this similarity difference. Specifically, we simulate 15 normal spots and 15 cancer spots in the ground truth labeling. Spatial labeling I includes six FNs, while spatial labeling II includes six FPs (\Cref{fig:case5}). The FPs and FNs are symmetrically distributed in the two labelings to eliminate topological variations. \Cref{tab:Table 1} shows that all benchmark metrics have zero \textit{Q} values, indicating their inability to capture the difference between the two labelings, as they focus solely on the number of mislabelings or the topological structure, both of which remain unchanged across the two labelings. In contrast, SLAM’s positive \textit{Q} value indicates its ability to recognize that the quality of spatial labeling I is inferior to that of spatial labeling II. This is because SLAM’s edge weight function (\Cref{severity weight}) assigns larger weights to edges between cancer spots, resulting in FNs being overweighted when measuring similarity with the ground truth.

\begin{figure}[h]
    \centering
    \captionsetup{belowskip=-10pt}
    \includegraphics[width=\linewidth]{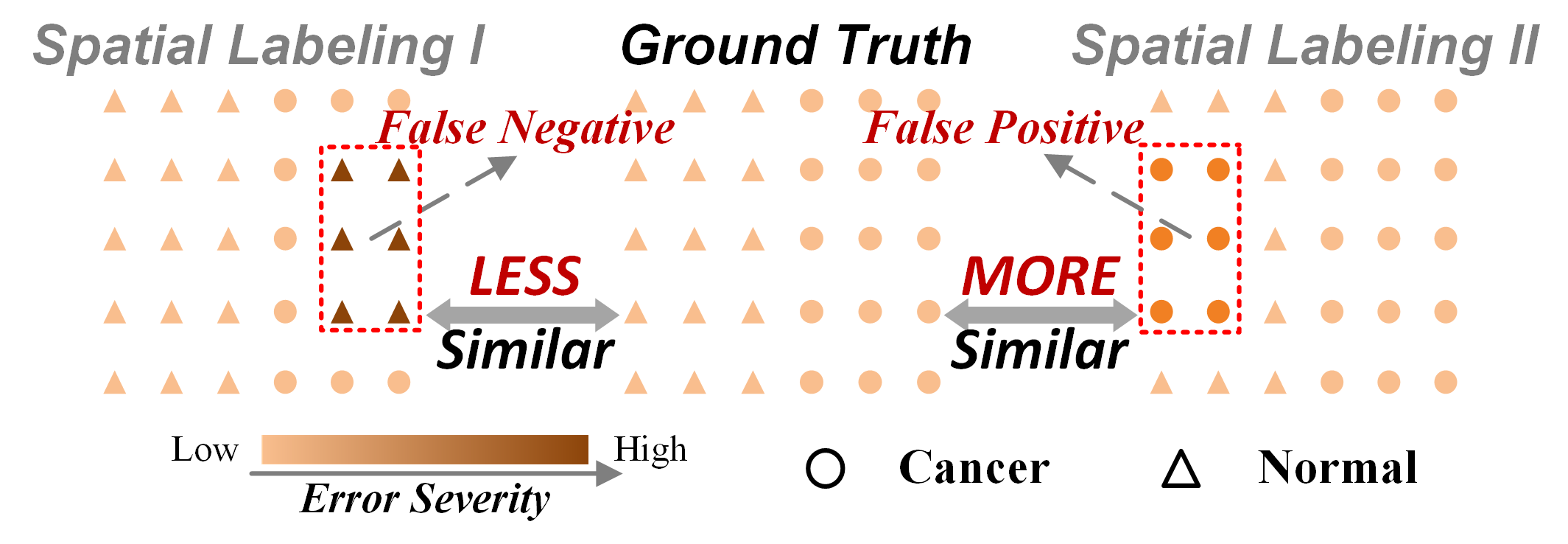}
    \caption{\textit{Case V}. Triangles represent normal spots and circles represent cancer spots. The color bar indicates the error severity level, with the leftmost color corresponding to error-free labels and the rightmost color to the most severe errors. 
    }
    \label{fig:case5}
\end{figure}

\paragraph{\textbf{SLAM differentiates mislabels with different similarities to true labels (\textit{Case VI}).}}  
In this case, the ground truth labeling comprises an equal number (10) of randomly selected spots from the adipose, breast gland, and cancer regions in the 10x-hBC-H dataset (\Cref{fig:case6}). Three adipose spots and three cancer spots are mislabeled as breast gland spots in spatial labeling I and II, respectively. The mislabel locations in the two spatial labelings are mirror-symmetric, ensuring that their label topological structure remains unchanged. Spots from the breast gland and adipose regions are more similar in gene expression, with an average normalized cosine similarity of 0.791, compared to 0.673 between spots from the breast gland and cancer tissues. Mislabeling breast gland spots as adipose spots leads to a labeling more similar to the ground truth than mislabeling them as cancer spots, thus a less severe error. Since spatial labeling I and II have an equal number of mislabels and identical topological structures, they differ only in mislabeling severity, which should be reflected by the evaluation metric. In \Cref{tab:Table 1}, all metrics except SLAM exhibit unchanged scores, as indicated by their zero \textit{Q} values, due to their insensitivity to mislabeling severity. SLAM, on the other hand, assigns edge weights based on the gene similarity between spots, thereby effectively capturing the labeling quality variation induced by this type of mislabeling severity.

\begin{figure}[h]
    \centering
    \captionsetup{belowskip=-10pt}
    \includegraphics[width=\linewidth]{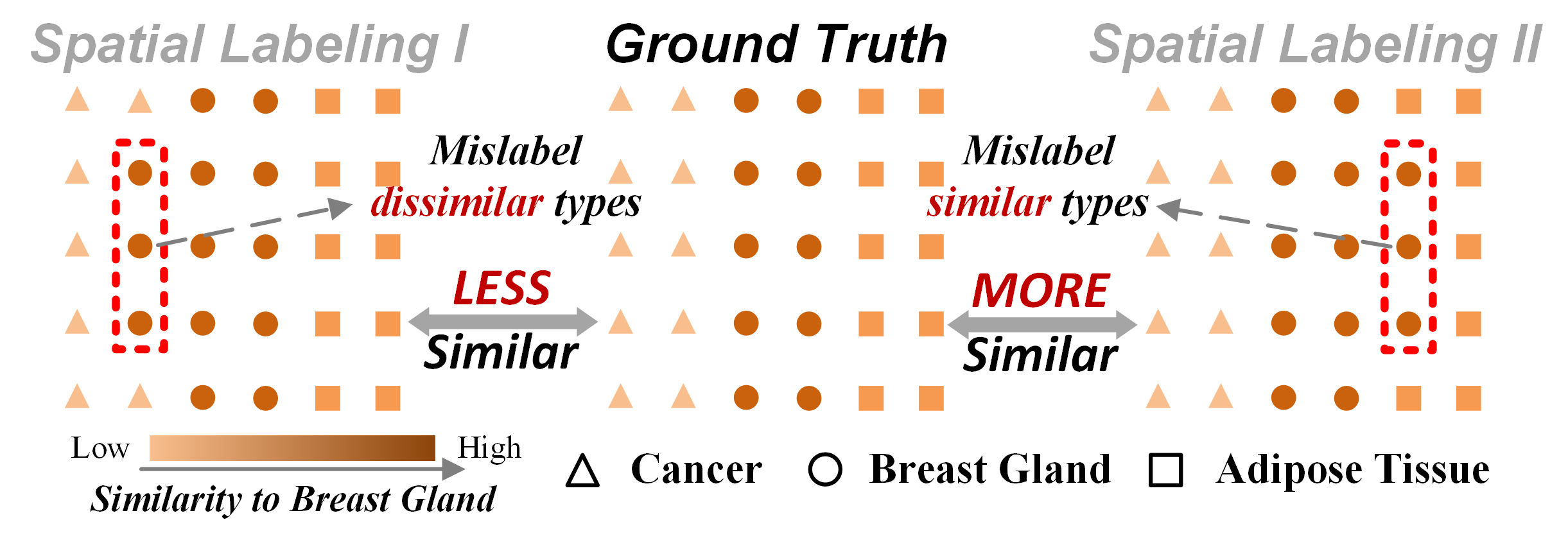}
    \caption{\textit{Case VI}. Triangles, circles, and squares represent adipose spots, breast gland spots, and breast cancer spots, respectively. The color bar indicates the similarity level in gene expressions with the breast gland tissue. 
 }
    \label{fig:case6}
\end{figure}

\subsection{Evaluating Spatial Labeling Results Using Real Spatial Transcriptomics Data (\textit{Case VII})}\label{Case_VII}

We evaluate SLAM’s effectiveness in assessing spatial labeling results using a real human breast cancer dataset (slice A1, 10x-hBC-A1)\cite{andersson2020spatial}. Since most spatial labeling methods for ST are unsupervised, we selected three well-established spatial clustering methods—SpaGCN, GraphST , and STAGATE  —to generate spatial labeling results. As these methods are unsupervised, we include only unsupervised and internal metrics as benchmarks. As shown in \Cref{fig:case7}, the results of GraphST and STAGATE visually resemble the ground truth more closely than that of SpaGCN. A closer comparison reveals that GraphST outperforms STAGATE, particularly in terms of mismatch severity—compared to GraphST, STAGATE generates significantly more FNs by mislabeling invasive cancer spots as connective tissue within the encircled region. \Cref{tab:Table 2} demonstrates that the unsupervised and internal benchmarks (except ASW) correctly reflect the inferior quality of SpaGCN compared to the other two clustering methods. However, SLAM stands out as the sole metric that identifies GraphST’s superiority over STAGATE. This is likely because SLAM concurrently considers label agreement, spatial organization, and mismatch severity,  offering a more comprehensive and unbiased evaluation.


\begin{figure}[h]
    \centering
    \captionsetup{belowskip=-10pt}
    \includegraphics[width=\linewidth]{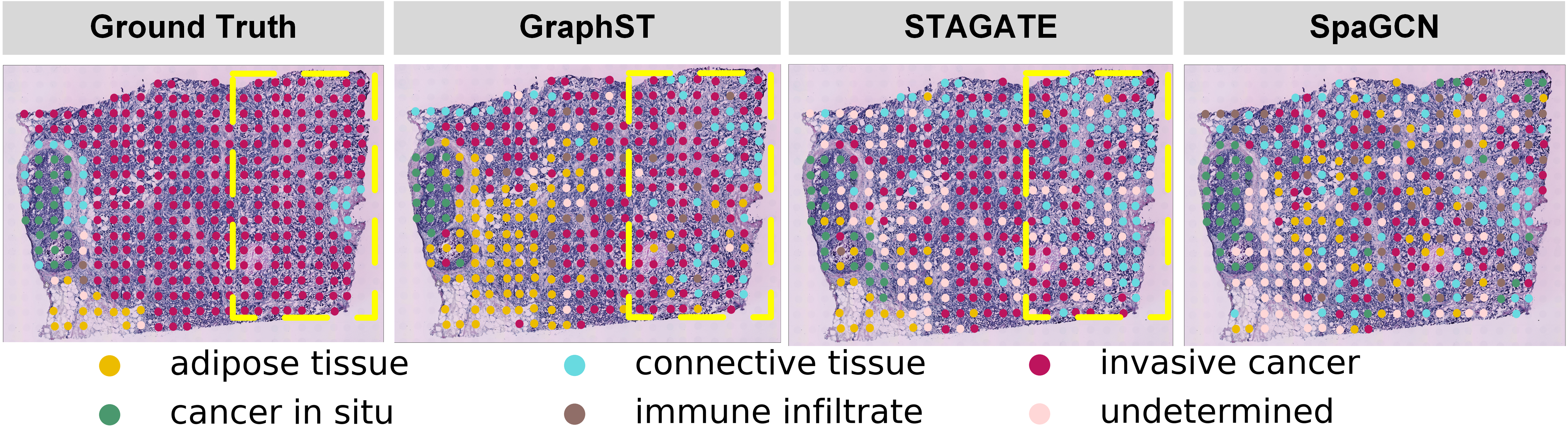}
    \caption{\textit{Case VII}. Three real clustering results of 10x-hBC-A1 datasets. GraphST outperforms STAGATE within the encircled region.}
    \label{fig:case7}
\end{figure}


\section{Conclusion}
We propose a methodological framework for measuring spatial labeling similarity. By accounting for all aspects of label agreement, spatial label distribution, and mismatch severity, our framework addresses the limitations of existing methods and provides a guideline for developing comprehensive and unbiased methods. Under this framework, we implement SLAM, a novel metric for evaluating spatial labeling result in ST based on its similarity to the ground truth labeling. SLAM exemplifies a concrete implementation the framework with detailed workflow steps. Through extensive carefully designed experimental cases involving both simulated and real ST data, SLAM demonstrates its superiority in accurately reflecting spatial labeling quality, highlighting the effectiveness of the proposed framework.

\newpage

\section*{Acknowledgments}
The project is funded by the Excellent Young Scientist Fund of Wuhan City (Grant No. 21129040740) to X.S.

\section*{Contribution Statement}
Yihang Du and Jiaying Hu contributed equally.

\bibliographystyle{named}
\bibliography{ijcai25}

\clearpage
\appendix
\onecolumn

\section*{\center \LARGE A Methodological Framework for Measuring Spatial Labeling Similarity}
\vspace{20pt}

\vspace{15pt}
{\centering \LARGE Appendix \par}

\vspace{20pt}

The appendix is organized as follows: In \textbf{Appendix \ref{Appendix_A}}, we provide more technical details of the framework and implementation, including the label matching function $\mathcal{M}$, and the theoretical foundation of the sliced Wasserstein distance and sliced Wasserstein Gaussian kernel. \textbf{Appendix \ref{Appendix_B}} gives the detailed calculation of the \textit{Q} coefficient. \textbf{Appendix \ref{Appendix_results}} includes the results and analysis of experimental cases II and IV. In \textbf{Appendix \ref{Complexity} and  \ref{sensitivity analysis}}, we perform the complexity analysis and sensitivity analysis on SLAM. \textbf{Appendix \ref{Metrics Introduction} }and\textbf{ \ref{The range and direction of benchmarks}} provide an overview of the benchmark evaluation metrics used in our experiments. \textbf{Appendix \ref{Appendix_fig}} visualize the implementation of our proposed framework as SLAM. Finally, the algorithmic pseudo-codes of SLAM are provided in \textbf{Appendix \ref{pseudocode}}.

\section{Additional Details on the Method} \label{Appendix_A}
\subsection{Implementations of $\mathcal{M}$ Function} \label{Appendix_M_function}
Assume $\mathcal{A}=\{a_1, a_2, \cdots, a_{K_1}\}$ and $\mathcal{K}=\{1, 2, \cdots, K\}$ represent the label spaces of a spatial labeling result and the ground truth, respectively. That is, there are $K_1$ distinct labels in $\mathcal{A}$ and $K$ distinct labels in $\mathcal{K}$. We first calculate Jaccard coefficients of all pairs of predicted and true clusters:
\begin{equation}
    J_{u,v}=\frac{|C_{u}\cap C_{v}|}{|C_{u}\cup C_{v}|}, \forall u \in \mathcal{A}, v \in  \mathcal{K}
\end{equation}
where $C_u$ denotes the set of spots in cluster $u$. Then we match $u$ to a true label $\iota\in \mathcal{K}$ as:

\begin{equation}
    u\mapsto\iota:=\{\max_v J_{u,v},|\ \forall J_{u,v}\in J_{u,\cdot}\}.
\end{equation}
For any true cluster $v \in \mathcal{ K}$, denote its set of assigned predicted clusters as $L_v:=\{u|u\mapsto v, \forall u\in \mathcal{A}\}$. If $\exists o\in \mathcal{K}, s.t. L_o=\emptyset$, the algorithm proceeds as:

\textbf{Cluster Reassignment}. When $K_1\ge K$, we reassign a predicted cluster $\tau\mapsto t$ to $o$ as:
\begin{equation}
\begin{aligned}
&\tau:=\{u|\max_u J_{u,o}, \forall J_{u,o}\in J_{\cdot,o}\}\mapsto o,\ if\\  
      &||L_{t}||>1\ and\ \tau\neq \{\max_u J_{u,t},|\ \forall J_{u,t}\in J_{\cdot,t}\}
\end{aligned}
\end{equation}
This step attempts to reassign an unmatched true label $o$ to a predicted label $\tau$, with which $o$ has maximal Jaccard coefficient, provided that $\tau$ does not already have the maximal Jaccard coefficient with its originally matched true label $t$. If this reassignment fails, $J_{\cdot,o}=J_{\cdot,o}\textbackslash{}J_{\tau,o}$, and the reassignment process is repeated until a satisfactory predicted label is identified. The process then continues to the next unmatched true label until $\nexists o\in \mathcal{K}\ s.t.\ L_o=\emptyset$. 

\textbf{Cluster Split}. When $K_1< K$, let $O=\{o|L_o=\emptyset, \forall o \in \mathcal{ K}\}$.  $\forall o\in O$, a new cluster $\phi$ is split from a predicted cluster $\tau\rightarrow t$ and assigned to it: 
\begin{equation}
    \tau:=\{u|\max_u J_{u,o}, \forall J_{u,o}\in J_{o,\cdot}\},
\end{equation}
\begin{equation}
    C_\phi:=\{c_n\in C_\tau|dist_{min}(c_n, C_o)<dist_{min}(c_n,C_t)\}\rightarrow o,
\end{equation}
where $dist_{min}(c_n, C)=\min_{\forall c_x\in C} dist(c_n,c_x)$. In this step, a predicted cluster $\tau$ that has the maximal Jaccard coefficient with $o$ is selected. Spots in $\tau$ that are located closer to the $o$ domain than to the $\tau$ domain form a new cluster $\phi$, which is then matched to $o$.  




\subsection{Definitions and Properties of Sliced Wasserstein Distance}\label{swd}
The sliced Wasserstein distance function measures the discrepancy between two sample distributions of graph attributes. It projects two multivariate distributions to be compared onto multiple random directions, obtaining their one-dimensional representations based on which one-dimensional wasserstein distance is calculated. The overall distance is computed as the integral of one-dimensional Wasserstein distances over the unit sphere. Since the calculation of one-dimensional Wasserstein distance has closed-form solution, sliced Wasserstein distance greatly improves the computational efficiency for multivariates \cite{bonneel2015sliced,kolouri2016sliced,rabin2012wasserstein}. 

\begin{definition}\label{Sliced WD} 
(\textbf{Sliced Wasserstein Distance}). Assume $\mathcal{P}$ and $\mathcal{Q}$ are two continuous d-dimensional probability measures with probability density functions $p$ and $q$, respectively. Let $v$ denote a random projection direction. The sliced Wasserstein distance between $p$ and $q$ is defined as:
\begin{equation}
    \begin{split}
        \mathcal{W}^2(\mathcal{P},\mathcal{Q}):=\underset{S^{d-1}}{\int}W^2(R(p(.,v)),R(q(.,v)))dv
    \end{split}
\end{equation}
where $R$ is the d-dimensional Radon transform function \cite{kolouri2016sliced} used to project probability densities of $\mathcal{P}$ and $\mathcal{Q}$ onto direction $v$. $W$ represents the $L2$ Wasserstein distance kernel. 
\end{definition}

\begin{definition}
    \label{CND}
    (\textbf{Conditionally Negative Definite Kernel}) A conditionally negative definite kernel is a symmetric function $k$: $\mathbb{S}\times \mathbb{S}\rightarrow \mathbb{R}$ that satisfies:
    \begin{equation}
        \sum_i^N\sum_j^Na_ia_jk(s_i,s_j)\le0,\forall s_i\in \mathbb{S}, \forall a_i\in \mathbb{R}\  \text{s.t.}\ \sum_i^Na_i=0.
    \end{equation}
\end{definition}
\begin{properties}\label{SW property}
\textbf{(Conditionally Negative Definiteness).} Sliced Wasserstein distance kernel is conditionally negative definite. 
\end{properties}
\begin{proof}
According to Brenier's theorem \cite{brenier1991polar,kolouri2016sliced}, $L2$ Wasserstein distance can be computed as:
\begin{equation}
    W^2(P,Q)=\int (\mathcal{T}_{\mathcal{P}\rightarrow\mathcal{Q}}(x)-x)^2q(x)dx,
\end{equation}
where $\mathcal{T}_{\mathcal{P}\rightarrow\mathcal{Q}}(x)$ is an unique optimal transport function that maps measure space $\mathcal{P}$ to $\mathcal{Q}$. It is monotonically increasing and satisfies:
\begin{equation}
 \mathcal{T}_{\mathcal{P}\rightarrow\mathcal{Q}}(x):=   \inf_{t\in \mathbb{R}} \{\int_{-\infty}^{t} p(t)dt \ge \int_{-\infty}^{x} q(x)dx\}
\end{equation}
We define an inner product space $\mathcal{I}$ over functions $f$, $g$: $\mathbb{R}\rightarrow \mathbb{R}$ as:
\begin{equation}
    \langle f,g\rangle _{\mathcal{I}}:=\int f(x)g(x)dx
\end{equation}
and a function:
\begin{equation}
    h_{\mathcal{P}\rightarrow\mathcal{Q}}(x):=(\mathcal{T}_{\mathcal{P}\rightarrow\mathcal{Q}}(x)-x)\sqrt{q(x)}
\end{equation}
Then we have:
\begin{equation}
\begin{aligned}
    &W^2(\mathcal{P},\mathcal{Q})=W^2(\mathcal{Q},\mathcal{P})=\langle h_{\mathcal{P}\rightarrow\mathcal{Q}},h_{\mathcal{P}\rightarrow\mathcal{Q}}\rangle_{\mathcal{I}},\\
    &W^2(\mathcal{Q},\mathcal{Q})=0,\\
    &W^2(\mathcal{P},\Pi)=W^2(\Pi,\mathcal{P})=\int (\mathcal{T}_{\Pi\rightarrow\mathcal{P}}(x)-x)^2 p(x)dx\\
    &=\int (\mathcal{T}_{\Pi\rightarrow\mathcal{P}}(\mathcal{T}_{\mathcal{P}\rightarrow\mathcal{Q}}(u))-\mathcal{T}_{\mathcal{P}\rightarrow\mathcal{Q}}(u))^2\underbrace{p(\mathcal{T}_{\mathcal{P}\rightarrow\mathcal{Q}}(u))\mathcal{T}'_{\mathcal{P}\rightarrow\mathcal{Q}}(u)}_{p(\mathcal{T}_{\mathcal{P}\rightarrow\mathcal{Q}}(u))d\mathcal{T}_{\mathcal{P}\rightarrow\mathcal{Q}}(u)=q(u)du}du\\
    &=\int (\mathcal{T}_{\Pi\rightarrow\mathcal{P}}(\mathcal{T}_{\mathcal{P}\rightarrow\mathcal{Q}}(u))-\mathcal{T}_{\mathcal{P}\rightarrow\mathcal{Q}}(u))^2q(u)du\\
    &=\int ((\mathcal{T}_{\Pi\rightarrow\mathcal{Q}}(u)-u)-(\mathcal{T}_{\mathcal{P}\rightarrow\mathcal{Q}}(u)-u))^2du\\
    &=\langle h_{\Pi\rightarrow\mathcal{Q}}-h_{\mathcal{P}\rightarrow\mathcal{Q}}, h_{\Pi\rightarrow\mathcal{Q}}-h_{\mathcal{P}\rightarrow\mathcal{Q}}\rangle_\mathcal{I}
\end{aligned}
\end{equation}
Then we have:
\begin{equation}
    \begin{split}
        &\sum_{i}^{N}\sum_{j}^{N}a_ia_jW^2(\mathcal{P}_i,\mathcal{P}_j)\\
        &=\sum_{i}^{N}\sum_{j}^{N}a_ia_j\langle h_{\mathcal{P}_i\rightarrow\mathcal{Q}}-h_{\mathcal{P}_j\rightarrow\mathcal{Q}}, h_{\mathcal{P}_i\rightarrow\mathcal{Q}}-h_{\mathcal{P}_j\rightarrow\mathcal{Q}}\rangle_\mathcal{I}\\
        &=\underbrace{2\sum_i^N a_i^2\langle h_{\mathcal{P}_i\rightarrow\mathcal{Q}}, h_{\mathcal{P}_i\rightarrow\mathcal{Q}}\rangle_\mathcal{I}}_{=0}-2\sum_i^N\sum_j^Na_ia_j\langle h_{\mathcal{P}_i\rightarrow\mathcal{Q}}, h_{\mathcal{P}_j\rightarrow\mathcal{Q}}\rangle_\mathcal{I}\\
        &=-2\langle \sum_i^N a_ih_{\mathcal{P}_i\rightarrow\mathcal{Q}},\sum_i^N a_ih_{\mathcal{P}_i\rightarrow\mathcal{Q}}\rangle_{\mathcal{I}}\le 0
    \end{split}
\end{equation}
Thus, for sliced Wasserstein distance, we have:
\begin{equation}
    \begin{split}
&\sum_{i}^{N}\sum_{j}^{N}a_ia_j\mathcal{W}^2(\mathcal{P}_i,\mathcal{P}_j)\\
        &=\sum_{i}^{N}\sum_{j}^{N}a_ia_j\underset{S^{d-1}}{\int} W^2(R(p_i(.,v)),R(p_j(.,v)))dv\\
        &=-2\underset{S^{d-1}}{\int} \langle \sum_i^N a_ih_{R(p_i(.,v))\rightarrow R(p_j(.,v))},\sum_i^N a_ih_{R(p_i(.,v))\rightarrow R(p_j(.,v))}\rangle_{\mathcal{I}}dv\\
        &\le 0
    \end{split}
\end{equation}
This completes the proof.
\end{proof}

\subsection{Definition and Properties of Sliced Wasserstein Gaussian Kernel}
\label{RKHS}
\begin{definition}
    \label{SWG def}
(\textbf{Sliced Wasserstein Gaussian Kernel}). Sliced Wasserstein Gaussian Kernel is a symmetric function $\kappa$: $\mathbb{S}\times \mathbb{S}\rightarrow \mathbb{R}$:
\begin{equation}
    \kappa(s_i, s_j) := \exp(-\psi \mathcal{W}^2(s_i, s_j)), \forall s_i\in \mathbb{S},\psi>0
\end{equation}
\end{definition}
\begin{properties}
    \label{SWG properties}
    (\textbf{Positive Definiteness and Inducibility of Unique RKHS}). Sliced Wasserstein Gaussian kernel is positive definite and induces a unique reproducible kernel Hilbert space (RKHS).
\end{properties}
\begin{proof}
As proved by C.Berg etc\cite{berg1984general,berg1984harmonic}: 
\begin{displayquote}
    ``A kernel $k(s_i,s_j)=\exp(-\psi f(s_i, s_j)),\psi>0$ is positive definite $\iff f$ is conditionally negative definite."
\end{displayquote}
Since $\mathcal{W}^2$ is conditionally negative definite ( \Cref{SW property}), $\kappa(s_i, s_j) = \exp(-\psi \mathcal{W}^2(s_i, s_j))$ is positive definite.Due to Moore-Aronszajn theorem \cite{aronszajn1950theory}, $\kappa$, as a symmetric and positive definite kernel, induces a unique RKHS. 
\end{proof}
\begin{proposition}
\label{SWG moments}
Sliced Wasserstein Gaussian kernel captures all moments of the squared sliced Wasserstein distance.
\end{proposition}
\begin{proof}
Using Taylor expansion, we have:
\begin{equation}
        \begin{split}
             \kappa(s_i, s_j) &= \exp(-\psi \mathcal{W}^2(s_i, s_j))\\
             &=\sum_i^{+\infty}\frac{1}{i!}(-\psi)^i(\mathcal{W}^2(s_i, s_j))^i,
        \end{split}
    \end{equation}
which is a linear combination of all moments of squared sliced Wasserstein distance.
\end{proof}

\section{Calculation of the \textit{Q} coefficient}\label{Appendix_B}
Given a spatial labeling evaluation metric, its $Q$ coefficient is computed as:
\begin{equation}\label{Evaluating metrics}
    Q = (-1)^n * \frac{s_{1}-s_{2}}{r},
\end{equation}
\begin{equation}
    n= \begin{cases}
    0, & \text{if larger metric score indicates lower labeling quality,} \\
    1, & \text{otherwise.}
\end{cases},
\end{equation}
\begin{equation}
    r= \begin{cases}
    \sup_{s}-\inf_{s}, & \text{if}\ \exists\ inf_{s}\ \text{and}\ \exists\ \sup_{s}, \\
    s_1-\inf_{s}, & \text{if}\ \exists\ inf_{s}\ \text{and}\ \nexists\ \sup_{s}, \\
    \sup_{s}-s_2, & \text{if}\ \exists\ \sup_{s} \text{and}\ \nexists\ \inf_{s}, \\
    max(|s_1|,|s_2|) & \text{otherwise.}
\end{cases},
\end{equation}

where $r$ is the value range of the metric, $s_1$ and $s_2$ represent the scores for a more erroneous and a less erroneous labeling result, respectively. A large positive \textit{Q} value indicates the target metric is both consistent (positiveness) and sensitive (large value change) to changes in spatial labeling quality.  

\section{Supplementary experimental cases} \label{Appendix_results}

Here, we design two additional experimental cases to demonstrate SLAM's effectiveness in evaluating the quality of spatial labeling in terms of label agreement degree and consistency in spatial label distribution.

\subsection{Label Agreement Degree}\label{Case:6}

\paragraph{SLAM correctly reflects labeling quality change due to increasing mislabels (\textit{case II}).}\label{case II} 
In this case, we simulate 360 spots, comprising 180 type A and 180 type B spots. We generate ten classification results, progressively increasing the number of A spots misclassified as type B, ranging from 9 to 95 (\Cref{fig:case2}a). Error rate ($1-accuracy$) is employed as an evaluation standard since it effectively measures labeling correctness. \Cref{fig:case2}b shows that SLAM increases with the number of misclassified spots, aligning with the trend observed in error rate. Thus, SLAM faithfully reflects the changes in spatial labeling similarity relative to the ground truth as the degree of label matching declines.

\begin{figure}[h]
    \centering
    \includegraphics[width=\textwidth]{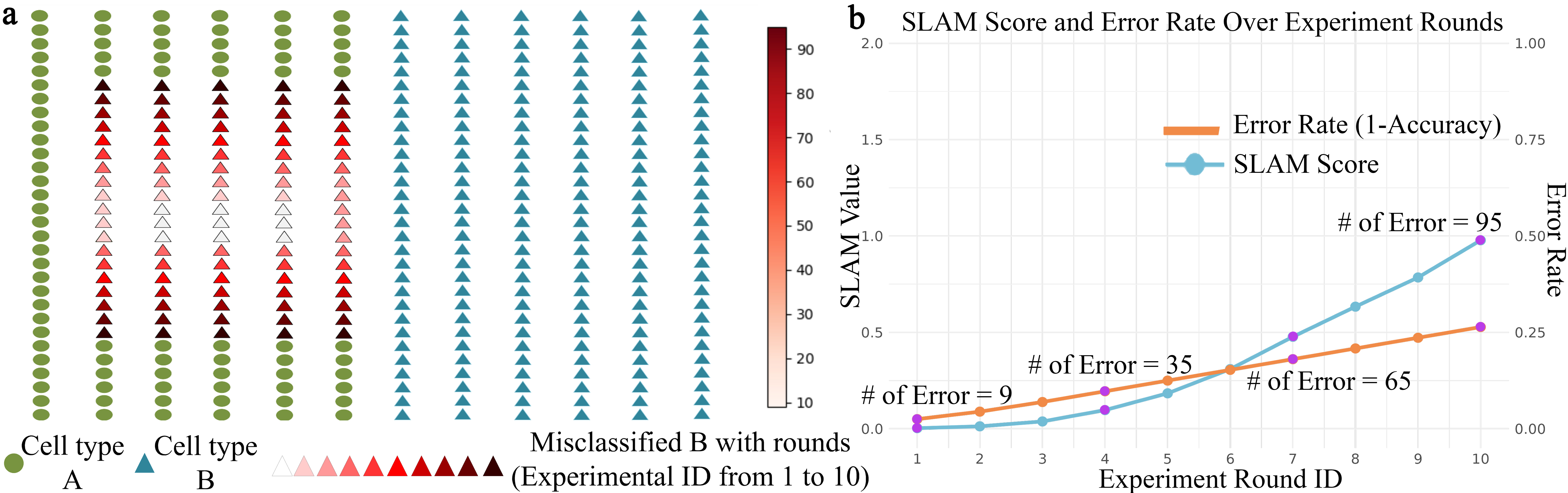}
    \caption{\textit{Case II}. SLAM reflects changes in similarity due to increasing mislabelings. a, 180 circle spots and 180 triangle spots represent distributions of type A and type B spots, respectively. The red spots represent type A spots misclassified as type B. The number of misclassified spots range from 9 to 95 in ten sequential experiments. A spot in deeper red is misclassified in later experiments.
    b, the trends of SLAM and error rate with increasing mislabelings.}
    \label{fig:case2}
\end{figure}


\subsection{Consistency in Spatial Label distribution}

\paragraph{SLAM differentiates aggregated versus dispersed mislabelings (\textit{case IV}).}
Aggregated and dispersed mislabels have distinct semantic meanings, resulting in different similarity levels with the ground truth, even when their numbers are equal. In this case, we aim to evaluate the effectiveness of SLAM and benchmarks in detecting the quality difference between two labelings with aggregated and dispersed mislabels. We simulate a dataset with 100 normal spots as the ground truth labeling (\Cref{fig:case4}). In spatial labeling I, 40 spots are mislabeled as cancer spots and spread throughout the entire tissue section, while in spatial labeling II, 40 mislabeled spots aggregate together. \Cref{tab:Table 1} shows that all supervised and unsupervised benchmark metrics remain unchanged across the two spatial labelings, as indicated by their zero \textit{Q} values. Conversely, SLAM and the five internal metrics demonstrate nonzero \textit{Q} values, indicating their ability to differentiate the quality of the two spatial labelings based on the level of mislabel aggregation.


\begin{figure}[h]
    \centering
    \includegraphics[width=0.5\linewidth]{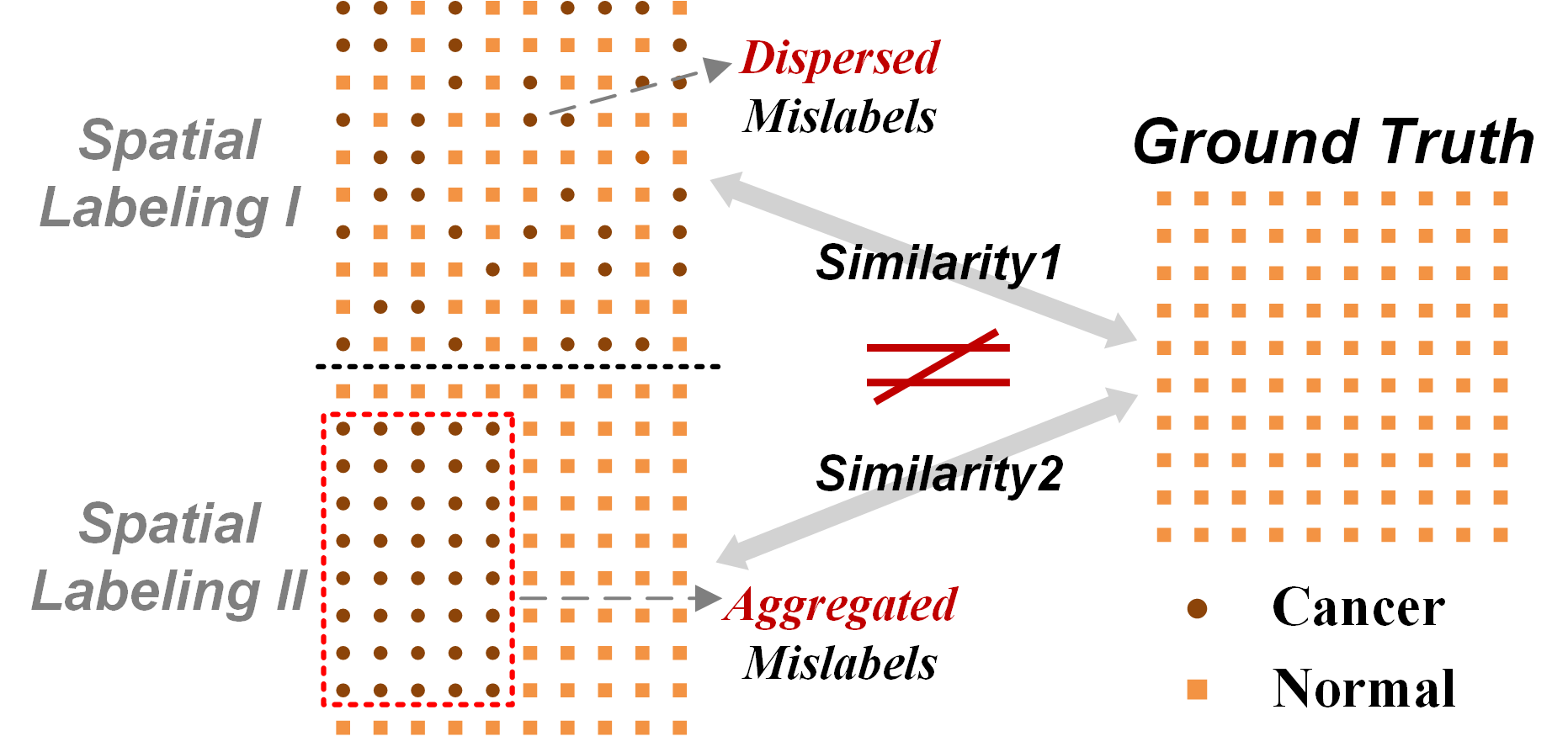}
    \caption{\textit{Case IV}. SLAM captures difference in similarity to the ground truth labeling between two spatial labelings with aggregated and dispersed mislabels.}
    \label{fig:case4}
\end{figure}

\begin{figure}[H]
    \centering
    \begin{subfigure}[t]{0.3\linewidth}
        \centering
        \includegraphics[width=\linewidth]{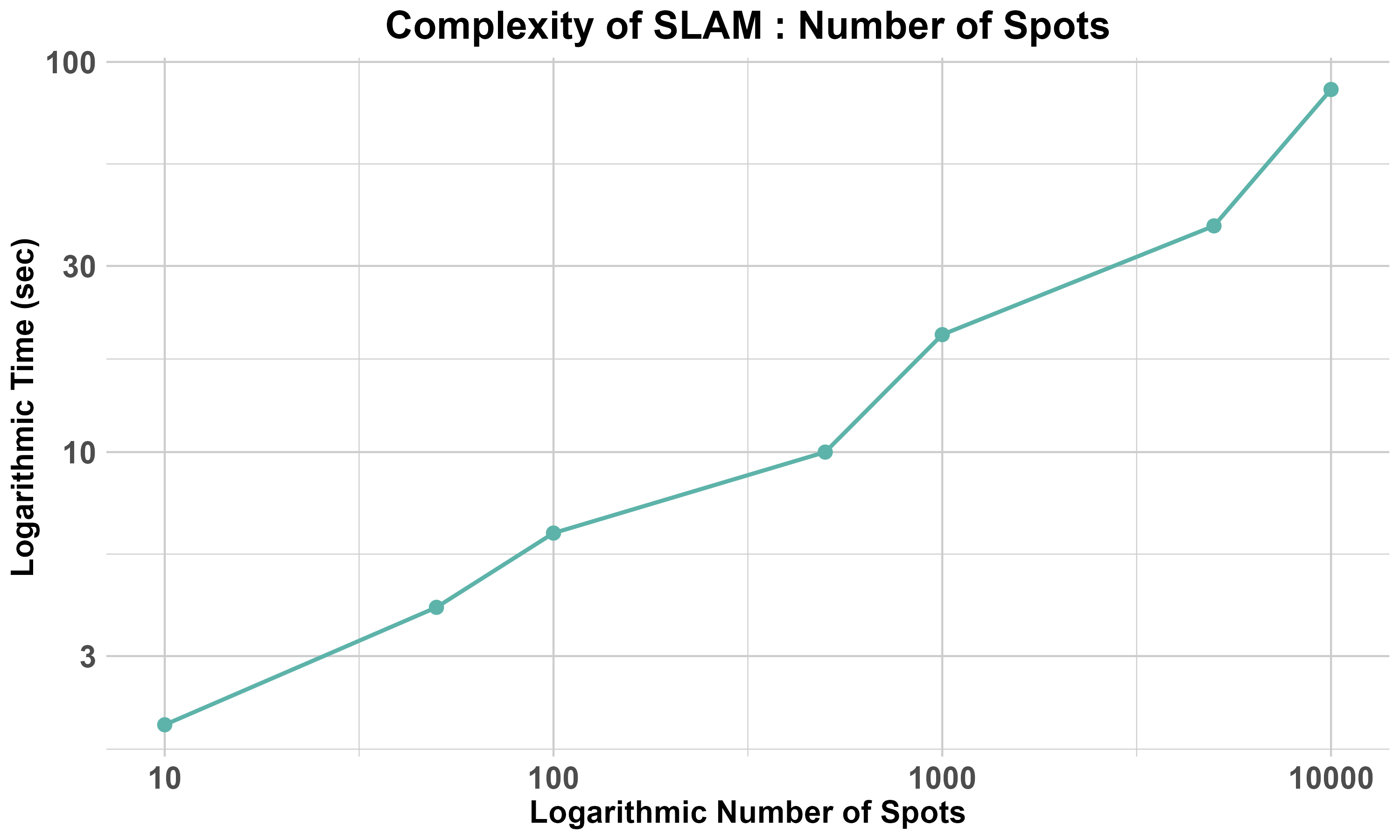}
        \caption{Complexity Analysis of SLAM with Respect to \textit{Number of Spots}.}
        \label{fig:complexity_spot}
    \end{subfigure}
    \begin{subfigure}[t]{0.3\linewidth}
        \centering
        \includegraphics[width=\linewidth]{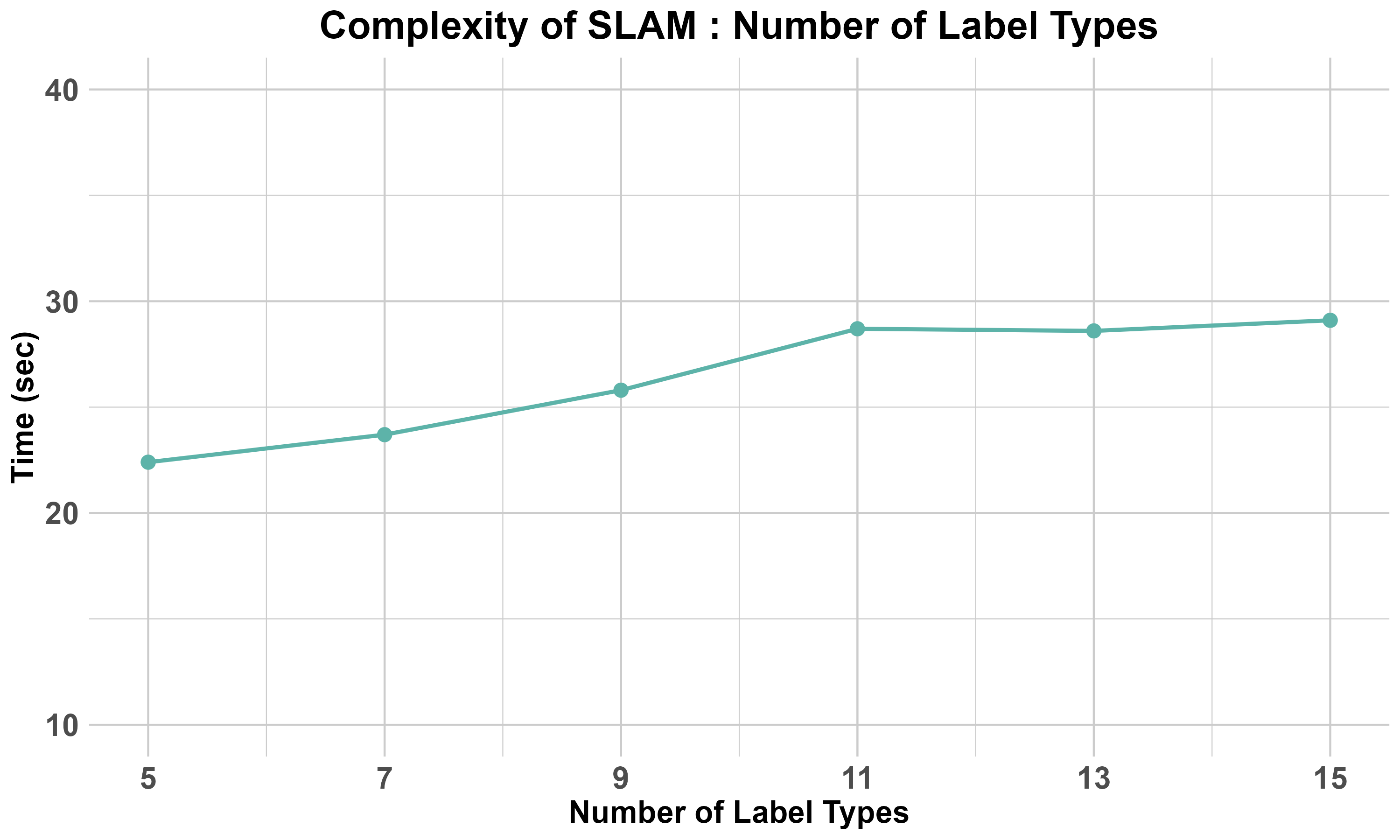}
        \caption{Complexity Analysis of SLAM with Respect to \textit{Number of Labels}.}
        \label{fig:complexity_label}
    \end{subfigure}
    \begin{subfigure}[t]{0.3\linewidth}
        \centering
        \includegraphics[width=\linewidth]{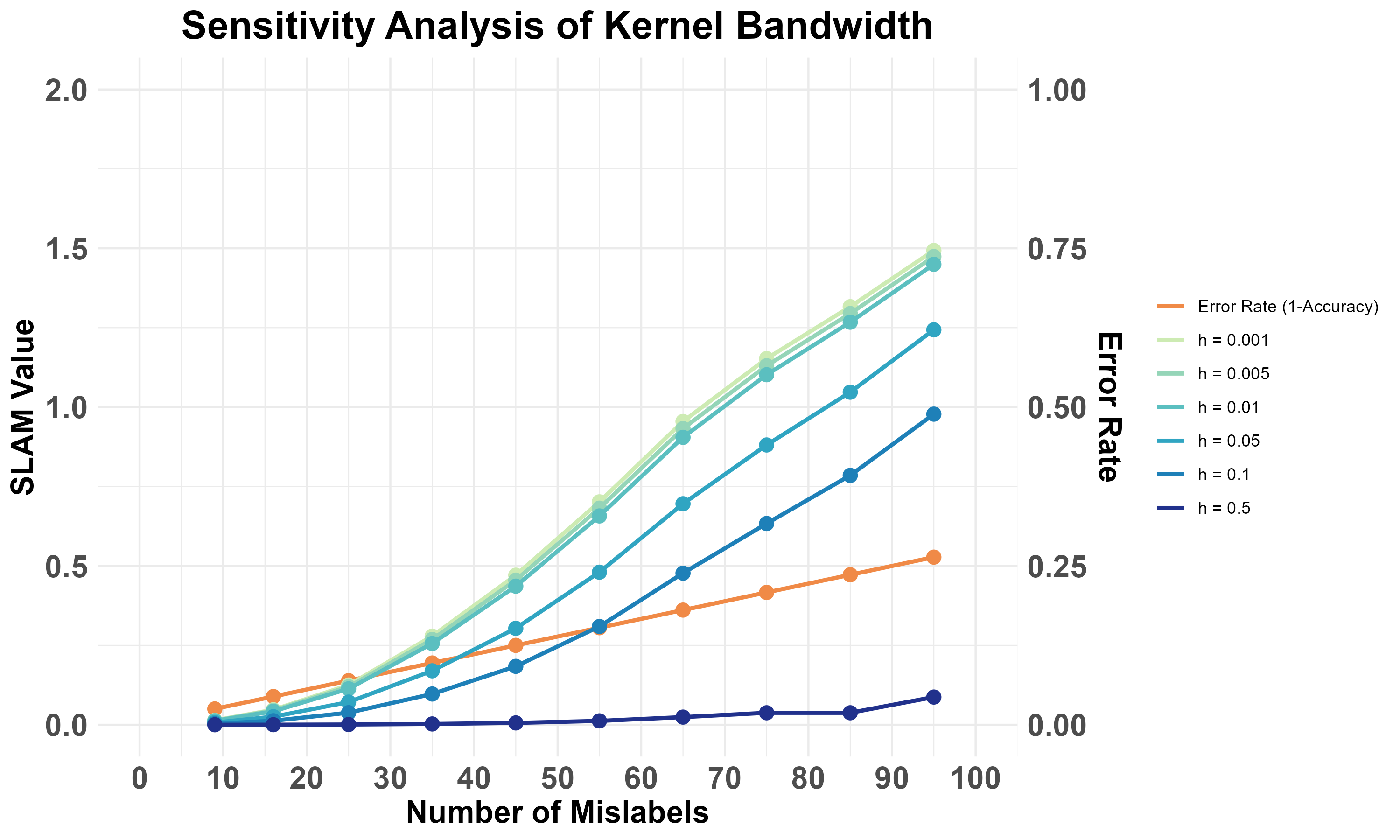}
        \caption{Sensitivity Analysis of \textit{Kernel Bandwidth}.}
        \label{fig:sensitivity}
    \end{subfigure}
\end{figure}

\section{Complexity Analysis}\label{Complexity}

To evaluate SLAM’s scalability to data size, we conduct a complexity analysis where the spatial labeling encompasses an increasing number of spots, ranging from 10 to 10,000. \Cref{fig:complexity_spot} shows that SLAM is scalable to data size, with a computation time of less than 100 seconds for evaluating 10,000 spots. Moreover, we evaluate SLAM's scalability to the number of label types. Using 1,000 spots, we start with five label types and incrementally increase the number of types by two until reaching 15, a relatively large value in ST studies. For each iteration, a spatial labeling result is generated by randomly mislabeling 50\% of the spots. As shown in \Cref{fig:complexity_label}, SLAM exhibits good scalability to the number of label types, with a computation time of less than 30 seconds for 15 types. Put together, these results highlight SLAM’s efficiency and practicality in spatial labeling evaluation.

\section{Sensitivity Analysis}\label{sensitivity analysis}

We conducted a sensitivity analysis on the bandwidth parameter $h$ in the Kernel Density Estimation (\Cref{KDE}). The analysis is carried out with our experimental case VI (\Cref{Case:6}), with $h$ values ranging from 0.001 to 0.5. Figure \ref{fig:sensitivity} shows the trend of SLAM values and the error rate with an increasing number of mislabels under different bandwidth values.  We find that the SLAM's performance is generally robust to the choice of bandwidth value (except when $h=0.001$), exhibiting a consistent trend with the error rate. Empirically, we suggest using 0.1 as the default value for $h$.

\section{Benchmarks Introduction}\label{Metrics Introduction}
We here provide a brief explanation of all the baseline metrics mentioned earlier, organized by category.

\begin{enumerate}
    \item \textbf{Supervised Metrics}\\
    \begin{itemize}
    \item Accuracy\cite{makridakis1993accuracy}:
    Measures the proportion of samples that are correctly classified out of the total number of samples. It is given by:
    $$
    \text{Accuracy} = \frac{\text{TP} + \text{TN}}{\text{Total Number of Samples}}
    $$
    \item Precision\cite{zhu2004recall}:
    Calculates the precision for each class and then takes the average, reflecting the classifier's ability to predict positive instances correctly. Precision is defined as:
    $$
    \text{Precision} = \frac{\text{TP}}{\text{TP} + \text{FP}}
    $$
    \item Recall\cite{zhu2004recall}:
    Calculates the recall for each class and then takes the average, reflecting the classifier's ability to identify all actual positive instances. Recall is defined as:
    $$
    \text{Recall} = \frac{\text{TP}}{\text{TP} + \text{FN}}
    $$
    \item F1 Score\cite{goutte2005probabilistic}:
    Calculates the F1 score for each class and then takes the average, which is the harmonic mean of precision and recall, reflecting the balanced performance of the classifier. The F1 score is given by:
    $$
    \text{F1 Score} = 2 \times \frac{\text{Precision} \times \text{Recall}}{\text{Precision} + \text{Recall}}
    $$
    \end{itemize}

    \vspace{1em}
    
    \item \textbf{Unsupervised Metrics}\\
    \begin{itemize}
        \item ARI (Adjusted Rand Index):
        ARI \cite{ARI} is a metric used to measure the similarity between two clustering results, adjusted for the chance grouping of elements. It compares the number of pairs of elements consistently assigned to the same or different clusters in both partitions, while correcting for expected agreements under random labeling. The ARI ranges from -1 to 1, where 1 indicates perfect agreement, 0 indicates random clustering, and negative values suggest worse-than-random clustering. Let $n$ represents the total number of spots, $n_{ij}$ the number of spots of type $i$ within cluster j, $a_i$ the total number of spots of type i, $b_j$ the total number of spots within cluster $j$. Then ARI is calculated as follows:
        \begin{equation}
        \label{E1}
        \mathrm{ARI} 
        = \frac{
          \displaystyle 
          \sum_{i,j} \binom{n_{ij}}{2} 
          \;
          \Bigl(\sum_{i} \binom{a_{i}}{2} \,\sum_{j} \binom{b_{j}}{2}\Bigr)\Big/\binom{n}{2}
        }{
          \displaystyle 
          \tfrac{1}{2}\Bigl[\sum_{i} \binom{a_{i}}{2} + \sum_{j} \binom{b_{j}}{2}\Bigr]
          \;-\;
          \Bigl(\sum_{i} \binom{a_{i}}{2} \,\sum_{j} \binom{b_{j}}{2}\Bigr)\Big/\binom{n}{2}
        }.
        \end{equation}



    \item NMI (Normalized Mutual Information):
    NMI \cite{NMI} is a metric used to evaluate the similarity between two clustering results by measuring the amount of shared information. It is based on mutual information and adjusts for differences in cluster sizes, making it robust for comparing partitions with varying numbers of clusters. NMI is normalized to range between 0 and 1, where 1 indicates perfect agreement between clusters and 0 indicates no shared information. The computation of NMI is defined as:
    \begin{equation}
    \label{E3}
    \text{NMI} = \frac{\text{MI(U,V)}}{\sqrt{H(U) \times H(V)}},
    \end{equation}
    where $H(U)$ and $H(V)$ denote the entropy of clusters $U$ and $V$, respectively, and $\mathrm{MI}(U, V)$ represents the mutual information between them. The mutual information is calculated as:
    \begin{equation}
    \label{E4}
    \mathrm{MI}(U, V) = \sum_{i=1}^{|U|} \sum_{j=1}^{|V|} 
   \frac{|U_i \cap V_j|}{N} 
   \log\biggl(\frac{N \, |U_i \cap V_j|}{|U_i| \, |V_j|}\biggr),
    \end{equation}
    where $N$ is the total number of elements, $|U_i|$ and $|V_j|$ represent the sizes of cluster $U_i$ and $V_j$, and $|U_i\cap V_j|$ is the number of elements shared between clusters $U_i$ and $V_j$.

    \item Jaccard Score:
    The Jaccard score \cite{JaccardScore} measures the similarity between true labels ($C$) and clustering results ($K$) by calculating the ratio of the size of their intersection $|C\cap K|$ to the size of their union $ |C\cup K|$ for each category. The final score is obtained by averaging this ratio across all categories. The Jaccard score ranges from 0 to 1, where 1 indicates perfect overlap between the true labels and clustering results, and 0 indicates no overlap. The Jaccard score is computed as:
    \begin{equation}
    \text{Jaccard Score} = \frac{1}{k} \sum_{l=1}^{k} \frac{|C_l \cap K_l|}{|C_l \cup K_l|},
    \end{equation}
    where $k$ denotes the total number of domain types, $C_l$ and $K_l $ represents the true and clustering labels of domain $l$, respectively. 

    \item FMI (Fowlkes-Mallows Index):
    The FMI \cite{FMI} evaluates the similarity between two clustering results by assessing the overlap of their pairwise assignments. It calculates the geometric mean of precision and recall based on true positive, false positive, and false negative pairs. The FMI ranges from 0 to 1, where 1 indicates perfect agreement between the two clusterings, and 0 represents no agreement. The FMI is computed as:
    \begin{equation}
    \mathrm{FMI} = \frac{\mathrm{TP}}{\sqrt{(\mathrm{TP} + \mathrm{FP}) \, (\mathrm{TP} + \mathrm{FN})}},
    \end{equation}
    where TP (True Positive) is the number of pairs of points that belong to the same cluster in both truth labels and predicted labels, FP (False Positive) is the number of pairs of points that belong to the same cluster in predicted labels but not in truth labels, and FN (False Negative) is the number of pairs of points that belong to the same cluster in truth labels but not in predicted labels. 

    \item V-measure:
    The V-measure\cite{v-measure}  is an external clustering evaluation metric that assesses the agreement between predicted clustering and ground truth by balancing two aspects: homogeneity and completeness. Homogeneity ensures that each cluster contains only data points from a single true class, while completeness ensures all data points of a true class are assigned to the same cluster. The V-measure is the harmonic mean of these two components, providing a score between 0 and 1, where 1 indicates perfect clustering. It is calculated as:
    \begin{equation}
    \text{V-measure} = \frac{2 \times h \times c}{h+ c},
    \end{equation}
    where $h = 1 - \frac{H\bigl(C \mid K\bigr)}{H(C)}$ is the homogeneity, $c = 1 - \frac{H\bigl(K \mid C\bigr)}{H(K)}$ is the completeness, $H\bigl(C \mid K\bigr)$ and $H\bigl(K \mid C\bigr)$ are the conditional entropies of the true labels $C$ and the predicted clusters $K$.
    \end{itemize}

    \vspace{1em}
    
    \item \textbf{Internal Metrics}\\
    \begin{itemize}
        \item Adjusted Silhouette Width (ASW):
        The ASW \cite{rousseeuw1987silhouettes} is a metric used to evaluate the quality of clustering by measuring how well a spot is assigned to its cluster compared to other clusters, while adjusting for chance. It is based on the Silhouette score, which considers the cohesion (distance within the same cluster) and separation (distance to the nearest other cluster). The adjustment in ASW accounts for differences in cluster sizes and random assignments. ASW values range from -1 to 1, where values close to 1 indicate well-separated and cohesive clusters, and values near 0 suggest random assignments. We first calculate Silhouette Width for spot $i$ as follows:
        \begin{equation}
        \label{E5}
        \mathrm{SW}_i = \frac{b_i - a_i}{\max(a_i, b_i)},
        \end{equation}
        where $a_i$ is the average distance to all other spots in the same domain, and $b_i$ is the average distance to all spots in the nearest other domain.
    
        \item Clustering Hierarchy of Anomalies and Outliers Score (CHAOS):
        The CHAOS metric can be applied to assess the spatial clustering performance by measuring the spatial continuity of identified domains \cite{CHAOS1,CHAOS2}. We first build a 1-nearest neighbor (NN) graph, wherein each spot is connected to its closest spot. Then we have:
        \begin{equation}
        \label{E7}
        w_{kij} = 
        \begin{cases}
        d_{ij}, & \text{if spot \(i\) and spot \(j\) are connected in the 1NN graph in cluster \(k\)},\\
        0, & \text{otherwise}.
        \end{cases}
        \end{equation}
        where $d_{ij}$ is the Euclidean distance between $\text{spots \(i\) and \(j\)}$. The CHAOS is then computed as the average value of $w$:
        \begin{equation}
        \label{E6}
        \mathrm{CHAOS} = \frac{\sum_{k=1}^{K} \sum_{i,j}^{n_k} w_{kij}}{N},
        \end{equation}
        where $n_k$ is the number of cells in the $k$-th domain, $N$ is the total number of spots, and $K$ is the number of domains. CHAOS values range from $0$ to $N/A$, with lower scores signifying greater spatial continuity and thus superior overall performance.
    
        \item Percentage of Abnormal Spots (PAS):
        The PAS metric \cite{PAS} evaluates the spatial uniformity of cluster labels within a ST field. It is defined as the fraction of spots whose cluster labels differ from at least six of their ten nearest neighbors. The PAS scores range from $0$ to $1$, with a smaller value indicating greater spatial homogeneity in cluster labels. 

        \item Calinski-Harabasz (CH) index:
        The CH index is a metric for evaluating clustering quality by measuring the ratio of between-cluster dispersion to within-cluster dispersion. It assesses how well-separated clusters are while ensuring that spots within each cluster are compact. CH Index is always positive, with a larger value indicating better-defined and more distinct clusters. Given a dataset of $N$ spots and $K$ clusters, CH index is calculated as:

        \begin{equation}
        \label{eq:CH_index}
        \mathrm{CH} 
        = \frac{
       \displaystyle 
       \frac{1}{K - 1}
       \sum_{k=1}^{K} n_k \,\bigl\lVert c_k - c \bigr\rVert^2
        }{
       \displaystyle 
       \frac{1}{N - K}
       \sum_{k=1}^{K}\sum_{i=1}^{n_k} \bigl\lVert d_i - c_k \bigr\rVert^2
        },
        \end{equation}
        where $n_k$ denotes the numbers of spots in the  $k^{th}$ cluster, $c_k$ the centroid of the $k^{th}$ cluster, and $c$ the global centroid.

    \item Davies-Bouldin (DB) index:
        The DB Index \cite{DB-index} evaluates clustering quality by measuring the average similarity between each cluster and the cluster most similar to it. It balances intra-cluster compactness and inter-cluster separation. The DB Index is always positive, with a lower value indicating better clustering, as it reflects well-separated and compact clusters. Given $K$ clusters, the DB index is calculated as:
        \begin{equation}
        \label{eq:DB_index}
        \mathrm{DB} = \frac{1}{K} \Sigma_{i=1}^K \underset{{j\neq i}}{max}( \frac{s_i + s_j}{d_{i,j}}),
        \end{equation}
    where $s_i$ denotes the average intra-cluster distance for cluster $i$, $d_{i,j}$ denotes the distance between the centroids of clusters $i$ and $j$.
    \end{itemize}
\end{enumerate}

\section{The Range and Direction of Benchmarks}\label{The range and direction of benchmarks}
\begin{table}[h]
    \centering
    \caption{The range and direction of benchmarks. $\uparrow$ means the metric score increases with improved performance, otherwise $\downarrow$.}
    \label{tab:Table 4}
    \begin{tabular}{cccc}
        \toprule
        Metric & Type & Range & Direction \\
        \midrule
        \textbf{\textit{SLAM}} & External & [0, 2] & $\downarrow$ \\
        \hline
        Accuracy & \multirow{5}{*}{Supervised external} & [0, 1] & $\uparrow$ \\
        Precision & & [0, 1] & $\uparrow$ \\
        Recall & & [0, 1] & $\uparrow$ \\
        F1 score & & [0, 1] & $\uparrow$ \\
        \hline
        NMI & \multirow{5}{*}{Unsupervised external} & [0, 1] & $\uparrow$ \\
        ARI & & [$-1$, 1] & $\uparrow$ \\
        Jaccard Score & & [0, 1] & $\uparrow$ \\
        V-measure & & [0, 1] & $\uparrow$ \\
        FMI & & [0, 1] & $\uparrow$ \\
        \hline
        CHAOS & \multirow{5}{*}{Internal} & (0, $\infty$) & $\downarrow$ \\
        PAS & & [0, 1] & $\downarrow$ \\
        ASW & & [$-1$, 1] & $\uparrow$ \\
        CH-Index & & (0, $\infty$) & $\uparrow$ \\
        DB-Index & & [0, $\infty$) & $\downarrow$ \\
        \bottomrule
    \end{tabular}
\end{table}

\newpage

\section{Framework Implementation} \label{Appendix_fig}
We provide a visual representation to better illustrate the framework implementation. 

\begin{figure*}[h]
    \centering
    \includegraphics[width=\textwidth]{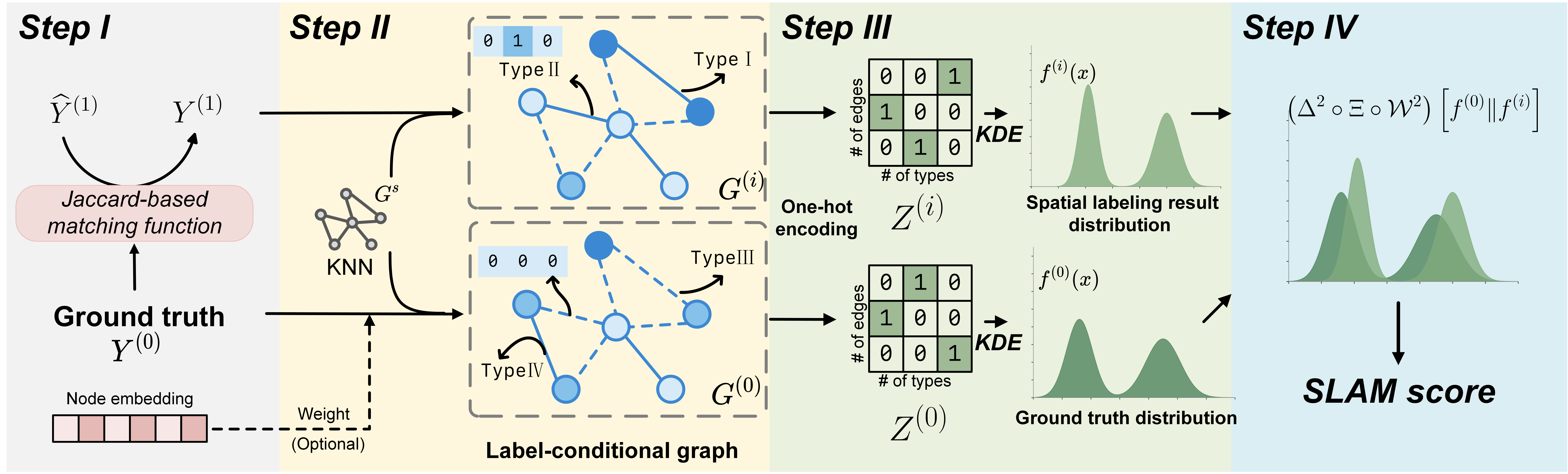}
  \caption{Framework Implementation. Initially, a Jaccard coefficient-based function is employed to match the label spaces of the two labelings. In the second step, SLAM edits the graph attributes based on spot types, locations, and attributes, incorporating information about label agreement, spatial label distribution, and mislabeling severity. In the third step, a gaussian kernel density estimation function is used to extract the distribution of the edited graph attributes. Finally, SLAM calculates a discrepancy score as an indicator of labeling dissimilarity using a composite function that integrates Wasserstein distance, exponential kernel, and MMD functions.}
    \label{fig:fig_implementation}
  \end{figure*}

\newpage

\section{The Algorithm of SLAM}\label{pseudocode}
\begin{algorithm}[H]
\caption{Implementation for evaluating spatial labeling using spatial transcriptomics (ST) data}
\begin{algorithmic}
\Require 
\Statex Data matrix $X \in \mathbb{R}^{n \times g}$ 
\Statex Ground truth labels $Y^{(0)} \in \{1,2, \cdots, K\}^n$  
\Statex Spatial labeling results $\hat{Y}^{(i)} \in \{a_1,a_2 \cdots, a_{K_i}\}^n$
\Statex Jaccard coefficient-based label matching function $\mathcal{M}$
\Statex Basic spatial information graph $G^s(V, E^s, W^s)$  
\Statex Graph attribute editing function $\mathcal{G}$
\Statex One-hot encoding extracting function $\mathcal{T}$
\Statex Distributional discrepancy function $\mathcal{D}$
\Ensure Discrepancy score $d$.
\end{algorithmic}

\begin{algorithmic}[1]
\State \textbf{Step 1: Label matching}
\State $Y^{(i)} \gets \mathcal{M}(\hat{Y}^{(i)}, Y^{(0)})$ \Comment{Match using Jaccard coefficient-based label matching function}

\State \textbf{Step 2: Constructing basic spatial graph}
\State Given the ST data basic spatial information graph $G^s(V, E^s, W^s)$:
\State For fixed number of neighbors $s$, edge set $E^s\in \{1,0\}^{|E^S|}$ is chosen in the s-nearest neighbor node set $v$.

\State For fixed radius $r$, edge weight set $W^s\in \mathbb{R}^{|E^S|}$ is 1, $\text{if } dist(u,v)\le r$ 

\State \textbf{Step 3: Constructing label-conditional attributed graph}
$G^{(i)}(V, E^{(i)}, W^{(i)}) \gets \mathcal{G}(Y^{(i)},G^s, X),\ \forall i\in \{0, 1, \cdots\}$ \Comment{Construct label-conditional attributed graph using graph attribute editing function}
\State The edge attributes $E^{(i)}$ are:
$$
e_{I(u,v)}^{(i)} =
\begin{cases}
t, & \text{if } y_{u}^{(i)} = y_{v}^{(i)} = t, e_{I(u,v)}^s = 1, \\
0, & \text{otherwise.}
\end{cases}
$$
where $t \in \{1, \cdots, K\}$, $e_{I(u,v)}^{(i)}\in E^{(i)}, e_{I(u,v)}^{s}\in E^{s}, \forall u,v\in V$. 
\State Gene similarity weight can be added accounting for severity variations across mislabel types.

\State \textbf{Step 4: Extracting distributions of graph attributes}
\State An One-hot encoding extracting function $\mathcal{T}$ is used to obtain the distribution of edited graph. Then SLAM use a Gaussian kernel density estimator to get a series of underlying distributions $f^{(i)}(x)$.

\State \textbf{Step 5: Computing the discrepancy between graph attribute distributions}
\State Distributional discrepancy function $\mathcal{D}:=(\Delta^2\circ \Xi\circ \mathcal{W}^2)$ (a composite function integrating the sliced Wasserstein distance function, a symmetric positive definite exponential kernel function $\Xi$, and a maximum-mean discrepancy (MMD) function $\Delta$)
\State Use $\mathcal{D}$ to compute discrepancy:
$$
\begin{aligned}
     & d = (\Delta^2\circ \Xi\circ \mathcal{W}^2)[ f^{(0)}|| f^{(i)}],
\end{aligned}
$$
\State \Return $d$
\end{algorithmic}
\end{algorithm}

\end{document}